\documentclass[journal]{IEEEtran}
\IEEEoverridecommandlockouts
\usepackage{cite}
\usepackage{amsmath,amssymb,amsfonts}
\usepackage{algorithmic}
\usepackage{graphicx}
\usepackage{textcomp}
\usepackage{xcolor}

\usepackage{tabularx}

\usepackage{mdwlist}
\usepackage[ruled,vlined,shortend, linesnumbered]{algorithm2e} 
\usepackage[]{algorithmic} 
\usepackage{booktabs}
\usepackage{color}

\usepackage{amsthm,amssymb}
\usepackage{amsmath}
\usepackage{multirow}
\usepackage{mathrsfs}
\usepackage{amsbsy}
\usepackage[mathscr]{eucal} 
\usepackage[flushleft]{threeparttable}
\usepackage[hidelinks]{hyperref}
\usepackage{comment}
\usepackage{paralist}
\usepackage{pifont}

\usepackage{bbding} 

\def\BibTeX{{\rm B\kern-.05em{\sc i\kern-.025em b}\kern-.08em
    T\kern-.1667em\lower.7ex\hbox{E}\kern-.125emX}}

\newcommand{\hide}[1]{}

\newcommand{\bit}{\begin{compactitem}}
\newcommand{\eit}{\end{compactitem}}
\newcommand{\ben}{\begin{compactenum}}
\newcommand{\een}{\end{compactenum}}

\newcommand{\method}{\textsc{DynWatch}\xspace}
\newcommand{\methodls}{\textsc{DynWatch-Local}\xspace}
\newcommand{\codeurl}{https://github.com/bhooi/dynamic.git}

\newcommand{\G}{\mathcal{G}}
\renewcommand{\S}{\mathcal{S}}

\newcommand{\V}{\mathcal{V}}
\newcommand{\E}{\mathcal{E}}
\newcommand{\N}{\mathcal{N}}
\newcommand{\datasmal}{\textsc{case2383}\xspace}

\newcommand{\dataxlarge}{\textsc{ACTIVSg25k}\xspace}
\usepackage{xcolor}

\usepackage[colorinlistoftodos]{todonotes}
\usepackage{xcolor}

\theoremstyle{definition}
\newtheorem{definition}{Definition}[section]

\usepackage{amsthm}
\usepackage{amsmath}
\usepackage{amssymb}
\usepackage{thmtools}
\usepackage{thm-restate}
\usepackage{hyperref}

\DeclareMathOperator{\Exp}{\mathbb{E}}

\newtheorem{assumption}{Assumption}[section]
\usepackage{mathtools}

\begin{document}

\title{Dynamic Graph-Based Anomaly Detection in the Electrical Grid
}

\author{Shimiao~Li,~\IEEEmembership{Graduate Student Member,~IEEE,}
	Amritanshu~Pandey,~\IEEEmembership{Member,~IEEE,}
	Bryan~Hooi,
	Christos~Faloutsos,
	and~Larry~Pileggi,~\IEEEmembership{Fellow,~IEEE}
\thanks{\copyright 2021 IEEE. Personal use of this material is permitted. Permission from IEEE must be obtained for all other uses, in any current or future media, including
reprinting/republishing this material for advertising or promotional purposes, creating new collective works, for resale or redistribution to servers or lists, or reuse of any copyrighted component of this work in other works.}
\thanks{S. L., A. P., and L. P. are with the Department of Electrical and Computer Engineering, Carnegie Mellon University, Pittsburgh, PA, 15213 USA (email:\{shimiaol, amritanp,pileggi\}@andrew.cmu.edu), B. H. is with the School of Computing and the Institute of Data Science in National University of Singapore(email:bhooi@comp.nus.edu.sg), C. F. is with the Department of Computer Science, Carnegie Mellon University, Pittsburgh, PA, 15213 USA (email:christos@cs.cmu.edu)}}

\maketitle

\begin{abstract}
	Given sensor readings over time from a power grid, how can we accurately detect when an anomaly occurs? 
A key part of achieving this goal is to use the network of power grid sensors to quickly detect, in real-time, when any unusual events, whether natural faults or malicious, occur on the power grid. 
Existing bad-data detectors in the industry lack the sophistication to robustly detect broad types of anomalies, especially those due to emerging cyber-attacks, since they operate on a single measurement snapshot of the grid at a time.
New ML methods are more widely applicable, but generally do not consider the impact of topology change on sensor measurements and thus cannot accommodate regular topology adjustments in historical data. 
Hence, we propose \method, a domain knowledge based and topology-aware algorithm for anomaly detection using sensors placed on a dynamic grid. 
Our approach is accurate, outperforming existing approaches by 20$\%$ or more (F-measure) in experiments; and fast, averaging less than 1.7ms per time tick per sensor on a 60K+ branch case using a laptop computer, and scaling linearly with the size of the graph.
\end{abstract}


\section{Introduction}
\label{sec:Introduction}
Maintaining and improving the reliability of the electric power grid is a critically important goal. Estimates~\cite{amin2011us} suggest that reducing outages in the U.S. grid could save $\$ 49$ billion per year, reduce emissions by $12$ to $18\%$, while improving efficiency could save an additional $\$ 20.4$ billion per year. Although grid operators and engineers work tirelessly to maintain reliability of the electric grids, many challenges persist. Climate change is increasing the frequency of natural disasters, resulting in higher rate of equipment failure. Adding to the climate risk is a new adversary in the form of cyber-intrusions that is capable of disrupting grid control and communication. This is evident from the recent reports of foreign hackers successfully penetrating control rooms of the U.S. power plants~\cite{perlroth2017hackers} and of cyber-attacks on the Ukrainian grid in 2015-2016 \cite{ukrain2015}\cite{lee2017crashoverride} that brought down sections of the network causing damages worth billions of dollars.  

A key tool that the grid operators use to safeguard against these failures, whether naturally occurring or malicious, involves the anomaly detection capabilities that are implemented in the grid control rooms. The primary purpose of these techniques is to help grid operators isolate faulty data from the healthy ones to result in accurate situational awareness, which further allows grid operators to take rapid corrective actions. In almost real-time, these methods can analyze measurement values, dynamics and other informative features to detect abnormal events including erroneous topology or measurements, while accommodating normal grid behaviors, including regular topology changes and power configuration adjustments.

In existing power grids, anomaly detection is performed within the Energy Management Systems (EMS) \cite{EMS} that are installed in the control rooms. The EMS through Supervisory Control and Data Acquisition (SCADA) system collects two primary sources of data: i) online analog measurement data from various sensors such as remote terminal units (RTUs) and phasor measurement units (PMUs) and ii) status data of switching devices and circuit breakers on various devices such as lines and transformers. 
Both types of data are collected every few seconds: for instance analog measurements are updated every 10s in PJM\cite{PJM-manual-01}; and status data are updated every 4s in ISO-NE\cite{ISO-NE-OP14}.
Upon processing, separate analysis units within the EMS are run to identify anomalies in measurements and topology.

AC state-estimation (ACSE) \cite{WLS-SE} along with bad-data detection (BDD) algorithms \cite{traditional-BDI} is used today for anomaly detection on measurement data from RTUs and PMUs, via hypothesis test on output residuals. These are run every 1 to 10 minutes (every 5 minutes in ERCOT and US Midwest ISO (MISO), every 3 minutes in the U.K. grid and ISO-NE, and every 1-2 minutes in PJM\cite{PJM-manual-12}). The most widely used problem formulation for ACSE is the weighted-least-square (WLS) form\cite{traditional-BDI}, minimizing mean-squared measurement error. Some other formulations are designed to achieve intrinsic robustness against bad data, including least absolute value (LAV) based \cite{LAVSE-PMU-abur}, least median of squares based \cite{robustSE-median}, as well as iteratively reweighted least-squares based approaches \cite{robustSE-reweight}. Unfortunately when RTUs are included, these methods generally suffer from difficult convergence due to non-linear measurement models. There have been recent attempts at convexification of the state-estimation problem \cite{convexSE-LAV-Li}\cite{convexTESE-SDP-weng}\cite{SUGAR-SE-Alex}\cite{SUGAR-SE-Li}; however, several limitations persist. These include inability to detect coordinated attacks like false-data injection attack and high sensitivity to network topology errors.

On the other hand, events of topology change are detected by the network topology processor (NTP)\cite{TE-NTP-book}\cite{TE-NTP-tracking}, which transforms the input circuit breaker/switching status data into the bus-branch model in which network connectivity and meter locations are identified. However, existing operational NTP does not account for topology errors due to erroneous status data caused by communication, operator entry errors, cyber-attacks, etc. As a countermeasure, there exist research works that overcome some challenges of existing NTP. For instance, generalized topology processing \cite{TE-GSE} creates pseudo-measurements and applies hypothesis tests to detect topology anomalies. Another approach \cite{TE-WLSE-BDI} applies hypothesis tests on WLS residuals from SE to detect topological errors, as an extended application of ACSE BDD. More recently, other advanced methods such as \cite{TESE-GSE}\cite{TESE-GSE-PMUabur}\cite{convexTESE-SDP-weng} have been developed where TE and ACSE are merged together to perform estimation on a node-breaker model. This enables measurement error and topology error to be effectively identified and separated. However, challenges 
persist here as well, mostly due to the lack of efficient and scalable methods to handle the non-linearity with conventional measurements and the inability of these methods to detect topology anomalies during coordinated cyber-attacks.

These challenges with existing TE and ACSE can be addressed by using anomaly detection based on statistical behavior.
Instead of analysing the well-defined measurement models in a snapshot, one can leverage statistical behaviors to extract some patterns from historical data. This can be helpful since most anomalies, either unexpected faults or cyber-attacks, usually disrupt the statistical consistency of the data stream, despite being invisible from one single snapshot. Existing behavioral anomaly detection methods can be broadly categorized into model-based (where an expectation of observation is obtained by fitting a mathematical model)\cite{hamilton1994time}\cite{regression4AD}\cite{yi2017grouped}, representation based\cite{breunig2000lof}\cite{LOFvariants}, graphical methods\cite{liu2008isolation}, and others (see Section \ref{sec:relatedwork} for related works). 
Generally, these families of methods are all directly applicable to the problem described in this paper. However, these methods either do not consider the impact of graph change on the observations \cite{hamilton1994time}\cite{yi2017grouped}\cite{breunig2000lof}, or assume a static topology across the data stream\cite{hooi2018gridwatch}. This is problematic for grid problems since  topological changes can happen frequently on a real grid (see Figure \ref{fig: topology and total power change}(a) in Section \ref{sec: discussion}), and different configurations naturally cause differences in grid measurements. Hence, previous observations from disparate topologies and loadings may provide little value in assessing the anomalousness at a given time \textit{t}. Without a selection of relevant data, methods may produce more false positives by falsely creating alarms when regular topology changes occur.

In light of these challenges, a more applicable way is to make the anomaly detection method \textit{context-aware}, so that it indicates the expected patterns of behavior from data samples collected with relevant context of topology. With that viewpoint, we propose a method that processes the time series of measurements and the time series of their topology to be \textit{context-aware}. Intuitively, the method works by defining graph distances based on domain knowledge and estimating a reliable distribution of measurements at time $t$ from the most relevant previous data. Then, an alarm is created if the measured value deviates greatly from the center of its distribution. The goal of the method is to directly consider the impact of topology on observations, so that regular topological changes are accommodated, while detecting any false measurement and topological errors.
Furthermore, to account for large-scale grids, we develop a locally-sensitive variant, \methodls.  


Figure \ref{fig:FDIA example} demonstrates the novelty of the proposed method. We apply a False Data Injection Attack\cite{fdia-review} to a 14-bus network, a special type of cyber attack that modifies the measurement output of the selected sensors in the grid. We show that our proposed method is able to detect the anomalies without false positive (FP) outcomes, while the baseline methods are not. Details of this experiment are provided in Section \ref{sec:experiment on FDIA}.

\begin{figure}[h]
	\centering
	\includegraphics[width=\linewidth]{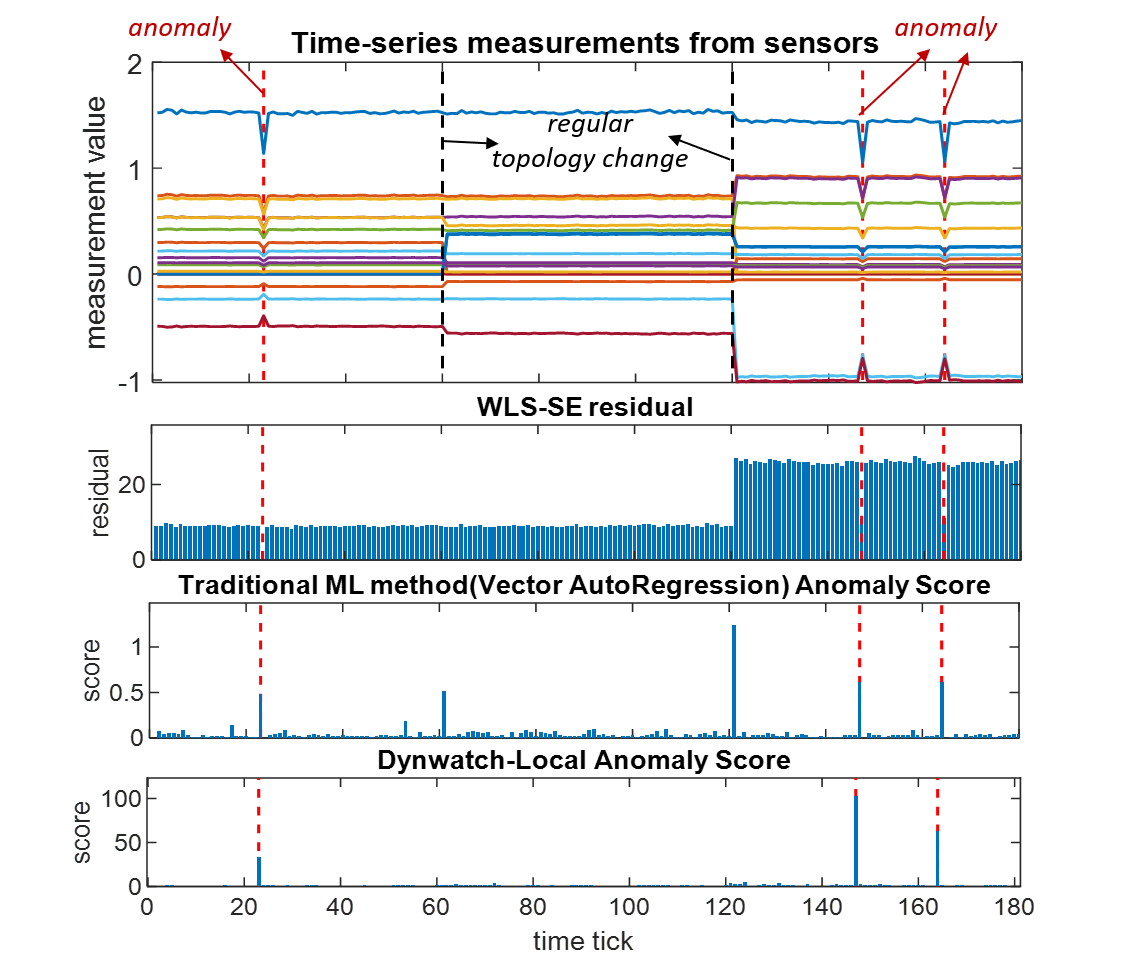}
	\caption[]{Example of FDIA: SE BDD fails to give high residuals; traditional ML method VAR misclassifies regular topology changes as anomalies; only Dynwatch is able to detect all anomalies without False Positives (FP).}
	\label{fig:FDIA example}
\end{figure}

The rest of the paper is organized as follows: Section \ref{sec:relatedwork} provides background of existing methods. Section \ref{sec:method} shows our proposed method. Section \ref{sec:Results} validates the method with experiment results. Finally, Section \ref{sec:Conclusion} concludes the paper.

\section{Background and Related Work} \label{sec:relatedwork}

\subsection{Existing ML-based Anomaly Detection Approaches}
\paragraph{Time Series Anomaly Detection} Numerous algorithms exist for anomaly detection in univariate time series~\cite{keogh2007finding}. For multivariate time series, LOF~\cite{breunig2000lof} uses a local density approach. Isolation Forests~\cite{liu2008isolation} partition the data using a set of trees for anomaly detection. Other approaches use neural networks~\cite{yi2017grouped}, distance-based~\cite{ramaswamy2000efficient}, and exemplars~\cite{jones2014anomaly}. However, none of these consider the graph structure. 

\paragraph{Anomaly Detection in Temporal Graphs} \cite{akoglu2010oddball} finds anomalous changes in graphs using an egonet (i.e. neighborhood) based approach, while \cite{chen2012community} uses a community-based approach. \cite{mongiovi2013netspot} finds connected regions with high anomalousness. \cite{araujo2014com2} detects large and/or transient communities using Minimum Description Length. \cite{akoglu2010event} finds change points in dynamic graphs, while other partition-based~\cite{aggarwal2011outlier} and sketch-based~\cite{ranshous2016scalable} also exist for anomaly detection. However, these methods focus on detecting unusual communities or connections, while our approach has a very different goal of detecting disturbances which cause changes in sensor values. 

\paragraph{Anomaly Detection with Domain Expert Knowledge} 
Domain-specific anomaly detectors based on optimal power flow\cite{anomaly-detection-OPF}, SE residual-based test (traditional SE BDI\cite{traditional-BDI}, and gross error detection\cite{gross_error_detection_1988}\cite{gross_error_detection_2015}) and TE\cite{TE-NTP-book}\cite{TE-GSE} already exist and are purely based on power system theories, yet they are typically limited to specific disturbances and attacks against the grid components. On the contrary, ML methods, as illustrated above, are more generally applicable; however, without a basic understanding of how the real-world grid operates, they are likely to perform poorly. Motivated by the pros and cons, many efforts \cite{hooi2018gridwatch}\cite{IDS-domain-knowledge} have combined the benefits of the two,  embedding the domain-knowledge in general ML methods. Such methods with domain knowledge have been shown to have higher performance (see \cite{hooi2018gridwatch}\cite{IDS-domain-knowledge}) but still do not fully consider the dynamic nature of the electric grid.


To summarize, 
the major contribution of \method when compared with existing methods are summarized in table \ref{tab:salesman}. 
{
\aboverulesep=0ex
\belowrulesep=0ex
\renewcommand{\arraystretch}{1.1}
\begin{table}[h!]
\small
\centering
\caption{Comparison of related approaches: only \method satisfies all the listed properties.}
\label{tab:salesman}
\begin{tabular}{@{}rccc|c@{}}
\toprule
 & 
 \rotatebox{90}{\textbf{Time Series}} & 
 \rotatebox{90}{\textbf{Graph-based}\ } & 
 \rotatebox{90}{\textbf{GridWatch}}  & 
 {\bf \rotatebox{90}{\method}} \\ \midrule
\textbf{Graph Data} &  & \Checkmark & \Checkmark  & \CheckmarkBold \\ 
\textbf{Anomalies in Sensor Data} & \Checkmark & & \Checkmark & \CheckmarkBold \\ 
\textbf{Changing Graph}  &  & \Checkmark &  & \CheckmarkBold \\
\textbf{Locally Sensitive}  &  & &  & \CheckmarkBold \\
\bottomrule
\end{tabular}
\end{table}
}

\subsection{Handling Redundant Data}Transmission grids, in order to be observable, have a large number of RTUs and PMUs installed. For the anomaly detection algorithm developed in this paper, processing the large volume of redundant data for anomaly detection is unnecessary and computationally prohibitive. This is because each sensor predominantly captures the relative information of its neighboring sensors as well. Therefore, to create a proper input for anomaly detection, pre-processing techniques can be deployed: Principal Component Analysis (PCA)\cite{PCA} creates a low-dimensional representation by extracting uncorrelated directions; projection pursuit\cite{projection-pursuit} reduces the input to a low-dimensional projected time series that optimizes the kurtosis coefficient; Independent Component Analysis (ICA)\cite{ICA} identifies a subset of independent variables. Alternative techniques like cross-correlation analysis\cite{cross-correlation-analysis} also help create a low-dimensional input. 
A detailed survey regarding dimensionality reduction can be found in\cite{review-dimension-reduction}. 
Rather than transforming the redundant input, other algorithms for sensor placement ~\cite{brueni2005pmu} are also applicable, by suggesting the best several locations of sensors to be installed and providing observability. \cite{hooi2018gridwatch} has shown the selection of a small number of sensor locations with a provably near-optimal probability of detecting an anomaly.

\subsection{Standard graph distance/similarity measures}\label{sec: graph distance survey}
Many graph distance/similarity measures have been proposed in the past that relate to anomaly detection in dynamic graphs. A survey of such measures can be found in \cite{graph_distance_survey2008} \cite{graph_distance_survey2015}. Most of these measures fit in one of the following categories:
\begin{enumerate}
    \item \textbf{Graph isomorphism and its generalizations:} examples include Maximum Common Subgraph (MCS) distance \cite{GED_MCS_shoubridge2002}, Graph Edit Distance (GED) \cite{GED_MCS_shoubridge2002}, and variants of GED\cite{GED_variants}, etc.
	\item \textbf{Aggregate statistical measure:} preferred for measure for larger graphs, examples include diameter distance 
	\cite{diameter_distance}, clustering based measures \cite{clustering_coeff}, and degree distribution\cite{degree_distribution}, etc.
	\item \textbf{Iterative methods based on the structural similarity of local neighborhoods:} This type of method exchanges node/edge similarities until convergence and computes the similarity between two graphs by coupling the similarity scores of nodes and edges \cite{graph_distance_survey2008}.
    \item \textbf{Other complex feature-based measures:} examples include graph kernel-based similarities \cite{graph_kernels}, modality distance\cite{modality_distance_bunke}, median graph distance\cite{modality_distance_bunke}, etc. 
\end{enumerate}

While all of the above graph distance measures have unique advantages, none of them are designed for grid-specific challenges, nor do they capture the implicit physics of the power grid graph. 
Section \ref{sec:measure} will provide a more detailed discussion about grid-specific challenges and develop a novel graph distance measure to meet the needs of the grid anomaly detection.

\subsection{Background: Line Outage Distribution Factor (LODF)}\label{sec:lodf}
Line Outage Distribution Factor (LODF) is a sensitivity measure of how much an outage on a line affects real power flow on other lines in the system~\cite{wood2013power}. This factor can be easily and efficiently calculated by assuming a DC power flow model with lossless lines or a linearized AC power flow model around the operating point and is commonly used to estimate the linear impact of line outage. For an outage on line $k$, LODF $d_{l}^{k}$ gives the ratio between power change $\Delta f_{l}$ on an observed line $l$ and the pre-outage real power $f_{k}$ on the outage line $k$.
\begin{align*}
d_{l}^{k}=\frac{\Delta f_{l}}{f_{k}}
\end{align*}

\section{Proposed \method Algorithm} \label{sec:method}

\subsection{Preliminaries} Table \ref{tab:dfn} shows the symbols used in this paper. 
\setlength{\tabcolsep}{6pt}
\begin{table}[htbp]
\small
\centering
	\caption{Symbols and definitions \label{tab:dfn}}
	\begin{tabular}{ @{}rl@{} }  
	\toprule
	\textbf{Symbol} & \textbf{Interpretation} \\ \midrule
		$\G=(\V,\E)$ & Input graph \\
		$\S$ & Subset of nodes to place sensors on\\
		$n$ & Number of nodes \\
		$s$ & Number of scenarios \\
		$\N_i$ & Set of edges adjacent to node $i$ \\
		$V_i(t)$ & Voltage at node $i$ at time $t$\\
		$I_e(t)$ & Current at edge $e$ at time $t$\\
		$s_{ie}(t)$ & Power w.r.t. node $i$ and edge $e$ at time $t$\\
		$\Delta s_{ie}(t)$ & Power change: $\Delta s_{ie}(t) =  s_{ie}(t) - s_{ie}(t-1)$\\
		$X_{i}(t)$ & Sensor vector for scenario $i$ at time $t$ \\
	\bottomrule
	\end{tabular} 
\end{table}

We are given a dynamic graph (grid) $\G_t=(\V(t),\E(t))$ at each time tick $t$, where $\V(t)$ denotes the set of nodes (active grid buses), and $\E(t))$ denotes the set of edges (active grid branches). Also we have a fixed set of sensors $\S \subseteq \V$. Each sensor installed on node $i$ can obtain PMU or RTU measurements
at each time $t$. 

For each sensor on node $i$, we obtain the power flows on all lines adjacent to node $i$, as observed in \cite{hooi2018gridwatch} that using power (rather than current) provides better anomaly detection in practice.
For any PMU bus $i$ and edge $e \in \N_i$, define the power w.r.t. $i$ along edge $e$ as $s_{ie}(t) = V_i(t) \cdot I_e(t)^*$, where $^*$ is the complex conjugate. 

{\bf Topology determination:} At any time $t$, the given dynamic graph $\G(t)$ can be considered a ‘reference topology’. The observed measurements should be consistent with the reference topology if there are no anomalous measurements and topology. In this work, we use the output from topology estimation (NTP) in the grid energy management system (EMS) as the reference topology. It is a valid reference because the operator believes that this is the best estimate of the grid topology at any time. Notably, this topology is not assumed to be perfect and accurate. Topological errors, along with false measurements, are to be detected by the proposed method.

\subsection{Motivation and Method Overview}

Consider the simple power grid shown in Figure \ref{fig:toyexample_nonlocal}, which evolves over time from $\G_1$ to $\G_2$ to $\G_3$. For simplicity, assume that we have a single sensor, from which we want to detect anomalous events. How do we evaluate whether the current time point ($t=15$) is an anomaly? If the graph had not been changing, we could simply combine all past sensor values to learn a distribution of normal behavior (e.g. fitting a Gaussian distribution as in $\mathcal{N}_{t=15}$), then evaluate the current time point using this Gaussian distribution (e.g. in terms of the number of standard deviations away from the mean). 
\begin{figure}[h]
	\centering
	\includegraphics[width=0.8\linewidth]{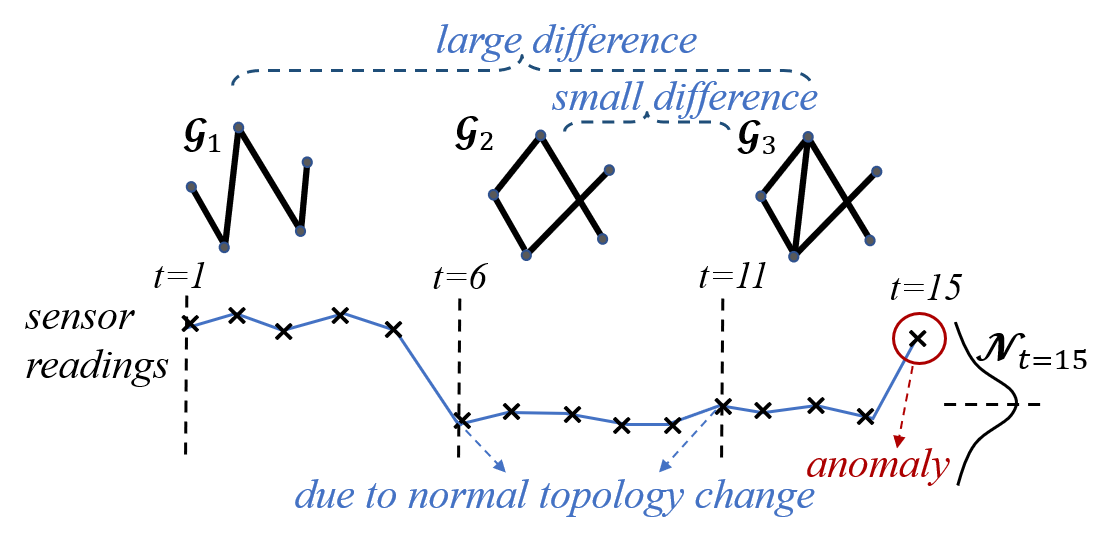}
	\caption[]{Simple motivating example: anomaly detection under changing graphs using \method.}
	\label{fig:toyexample_nonlocal}
\end{figure}

In the changing graph setting, we still want to learn a model of normal behavior ($\mathcal{N}_{t=15}$), but while taking the graph changes into account. Note that $\G_2$ and $\G_3$ are only slightly different, while $\G_1$ and $\G_3$ are very different. Hence, the sensor values coming from $\G_2$ (i.e. time $6$ to $10$) should be taken into account more highly when constructing $\mathcal{N}_{t=15}$, as compared to those from $\G_1$. Intuitively, the sensor values from $\G_1$ are drawn from a very different distribution from the current graph, and thus should not influence our learned model $\mathcal{N}_{t=15}$. In general, the more similar a graph is to the current graph, the more we should take its sensor values into account when learning our current model. This motivates the 3-step process we use:
\begin{enumerate}
\item {\bf Graph Distances:} Measure the distance between each past graph and the current graph.
\item {\bf Temporal Weighting:} Weight the past sensor data, where data from graphs that are similar to the current one are given higher weight.
\item {\bf Anomaly Detection:} Learn a distribution of normal behavior ($\mathcal{N}_{t}$) from the weighted sensor values, and measure the anomalousness at the current time based on its deviation from this distribution.
\end{enumerate}

To further clarify the motivation and methodology, we provide an informal definition for the anomaly detection problem and a statistical definition of the anomaly pattern. 

\begin{definition}[Dynamic electric grid anomaly detection problem]
Given time series data of sensor observations on a set of sensors $\{s\}$, and a time series of topologies, find (1) the timestamps that correspond to an anomaly pattern (2) the top-k sensor locations that contribute most to the anomaly pattern.
\end{definition}
\begin{definition}[Anomaly pattern]At any time $t$, given a time series of observations $X(t),X(t-1),…X(t-W)$, a time series of topologies, a graph distance measure $D(Gi,Gj)$, and a detection threshold $\tau$, we assign different (trust) weights to observations at $t-1,…, t-W$, based on $D(G_t,G_{(t-1)} ),…,D(G_t,G_{(t-W)})$: the larger the distance, the lower the weight. This provides a statistical distribution of $X(t)$, parameterized by weighted median $\mu(t)$ and weighted IQR(t). The instance $t$ is predicted as anomalous if $X(t)$ is an outlier of the distribution, i.e., $|X(t)-\mu(t)|>\tau\cdot IQR(t)$.
\end{definition}
In electric grids, the target anomalies correspond to topological errors (unexpected topology changes unknown to the operator) and measurement errors. These are the types of anomalies that our method is proposed for and is likely to detect based on empirical results provided at the end of the manuscript.

In the following sections, we first introduce our domain-aware graph distance measure based on Line Outage Distribution Factors (LODF)~\cite{wood2013power}. Then, we describe our temporal weighting and anomaly detection framework, which flexibly allows for any given graph distance measure. Finally, we present an alternate distance measure that is locally sensitive, i.e., it accounts for the local neighborhood around a given sensor.

\subsection{Proposed Graph Distance Measure} \label{sec:measure}

In this section, we describe our proposed graph distance measure to calculate the distance $D(\G_i, \G_j)$ between any pair of graphs. For ease of understanding, the rest of Section \ref{sec:method} uses the example of anomaly detection at $t=15$ in Figure \ref{fig:toyexample_nonlocal} as an extended case study, but our approach can be easily extended to the general case. 

Section \ref{sec: graph distance survey} has provided a short review of existing graph distance measures, but these measures do not apply directly to the grid-specific challenge in this paper. For an anomaly detection algorithm to work well on the power grid applications, the choice of graph distance needs to consider {\bf problem-specific challenges} along with desirable properties for an anomaly detection algorithm (scalability, sensitivity to change, and {\bf ‘importance-of-change awareness’}). One grid-specific challenge is that the ideal graph distance should capture ‘grid physics’ rather than only graph structural changes. Specifically, the distance should be sensitive to the redistribution of power flow, not only the addition/deletion of nodes/edges, since the anomaly information is extracted from power flow measurements. Meanwhile, as the ‘importance-of-change awareness’ indicates, grid changes that cause big shifts in power flow (measurements) should result in larger graph distances, than changes that cause minor power impact. Unfortunately, none of the classical graph distances can capture the physics of the power flow and quantify the impact of graph change in terms of power. To handle this, this work proposes a novel design of graph distance by making use of the power sensitivity factor.

Intuitively, the goal is for our graph distance to represent {\bf redistribution of line power flow}. Critical changes in topology result in large redistributions of power. Thus, the graph distance arising from a topology change should be large if the changed edges can potentially cause large amounts of power redistribution.

Hence, given two graphs $\G_i(\V(i),\E(i))$ and $\G_j(\V(j),\E(j)$ with different topology, we first define a transition state which takes the union of the two graphs:
$$\G_{trans}=(\V(i)\cup\V(j),\E(i)\cup\E(j))$$ 

\begin{figure}[h]
	\centering
	\includegraphics[width=0.4\linewidth]{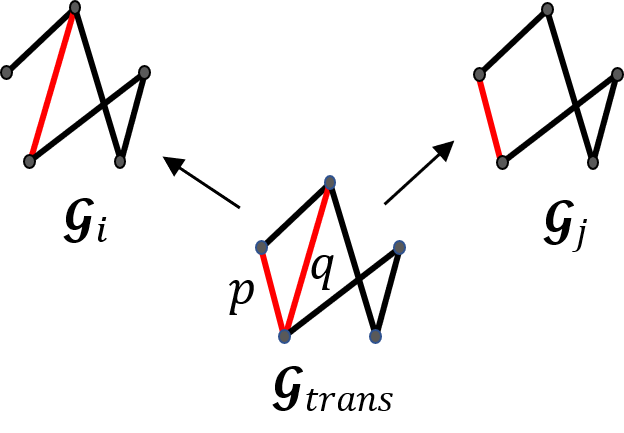}
	\caption[]{Transition state of two graphs: the union of two graphs.}
	\label{fig:graphdistance}
\end{figure}

Then the topology changes from $\G_i$ to $\G_j$ can be considered as different line deletions from their base graph $\G_{trans}$. For each single line deletion, e.g. line $p$, we define its contribution $x_p$ to graph distance by taking the average of its power impacts on all other lines as measured by LODF:
$$x_p = \frac{1}{|\E(i)\cup\E(j)|} \sum_{l\in \E(i)\cup\E(j)\backslash\{p\}}(|d_l^p|)$$
where $|\E(i)\cup\E(j)|$ denotes the cardinality of set $\E(i)\cup\E(j)$, $d_l^p$ denotes the LODF coefficient with $p$ as outage line and $l$ as observed line. 

Then graph distance $D(\G_i, \G_j)$ is given by summing up the contributions of different line deletions from the base graph:
$$D(\G_i, \G_j) = \sum_{p\in (\E(i)-\E(j)) \cup (\E(j)-\E(i))}x_p$$ 
where $\E(i)-\E(j)=\{p|p\in \E(i),p\notin \E(j)\}$ and accordingly, $(\E(i)-\E(j)) \cup (\E(j)-\E(i))$ denotes all the edge changes between the two graphs.

This definition uses LODF as a measure of the impact on power flow of the removal of line $p$. Hence, edges with high LODF to many other edges can potentially cause greater changes in power flow, and thus our graph distance measure places greater importance on these edges. Appendix \ref{apdx: compare distances} demonstrates the effectiveness of the LODF-based graph distance by comparing it against traditional distance measures for anomaly detection.

\subsection{Proposed Temporal Weighting Framework} \label{sec:framework}

In this section, we assume that we are given any distance measurements $D(\G_i, \G_j)$ between any pair of graphs $\G_i$ and $\G_j$, and explain how to use them to assign weights to each previous sensor data. This procedure can take the LODF-based distance defined in the previous subsection as input, but also allows us to flexibly use any given graph distance measure. The proposed Temporal Weighting is given in Algorithm \ref{alg:framework}.

\begin{algorithm}
	\caption{Temporal Weighting Framework at time $t=15$ (see example in Figure \ref{fig:toyexample_nonlocal})} 
	\label{alg:framework}
	\KwIn{Graph distance $D(\G_1, \G_3)$, $D(\G_2, \G_3)$, $D(\G_3, \G_3)$; sensor data $s_i(t)$ with $t=1,2,...15$, $ i=1,2,..., N_{sensor}$.}
	\KwOut{Anomaly score $A(15)$.}
	{\bf Extend graph distance to tick-wise distance.} Each previous time tick is given a distance $d_t$ according to the graph it comes from:
	$$d_t = \begin{cases}
	D(\G_1, \G_3) & \text{for } t = 1,2,..., 5 \\
	D(\G_2, \G_3) & \text{for } t = 6,7,..., 10 \\
	D(\G_3, \G_3) & \text{for } t = 11,..., 14
	\end{cases}$$
	
	{\bf Temporal Weighting:} Use $d_1,...d_{14}$ to assign weights $w_1$, ..., $w_{14}$ to the past sensor data using Algorithm \ref{alg:weighting}.
\end{algorithm}

For the purpose of utilizing previous data from a series of dynamic graphs, {\bf Temporal Weighting} plays an important role. The resulting weights directly determine how much information to extract from each previous record, thus requiring special care. Intuitively, the weights should satisfy the following principles:
\bit
\item The larger the distance $d_t$, the lower the weight $w_t$. This is because high $d_t$ indicates that time tick $t$ is drawn from a very different graph from the current one, and thus should not be given high weight when estimating the expected distribution at the current time
\item Positivity and Normalization: $\sum_{t}w_t = 1, w_t\geq0$
\eit
To satisfy these conditions, we use a principled optimization approach based on {\bf bias-variance trade-off}. Intuitively, the problem with using data with high $d_t$ is {\bf bias}: it is drawn from a distribution that is very different from the current one, and that can be considered a biased sample. We treat $d_t$ as a measure of the amount of bias. Hence, given weights $w_1, \cdots, w_{14}$ on previous data (in Figure \ref{fig:toyexample_nonlocal} example), the total bias we incur can be defined as $\sum_{t\in\{1,...,14\}}w_t d_t$. 

We could make the bias low simply by assigning positive weights to only time points from the most recent graph. However, this is still unsatisfactory as it results in a huge amount of {\bf variance}: since very little data is used to learn $\mathcal{N}_{t=15}$, the resulting estimate has high variance. Multiplying a fixed random variable by a weight $w_t$ scales its variance proportionally to $w_t^2$. Hence, given weights $w=[w_1,w_2,\cdots]$, the total amount of variance is proportional to $\frac{1}{2}w^Tw$, which we define as our variance term.

We thus formulate the following optimization problem as minimizing the sum of bias and variance, thereby balancing the goals of low bias (i.e. using data from similar graphs) and low variance (using sufficient data to form our estimates). We formulate the problem as:
\begin{align*}
\min_{w}{\sum_{t}w_t d_t+\frac{1}{2}w^Tw}
\end{align*}
subject to
\begin{align*} 
&\sum_{t}w_t = 1\\
&w_t\geq 0,\forall t
\end{align*}

By writing out its Lagrangian function:
\begin{align*}
L(w,\lambda,u)=d^Tw+\frac{1}{2}w^Tw+\lambda(1-\sum_t w_t)-u^Tw
\end{align*}
 and applying KKT conditions, we can see the optimal primal-dual solution $(w,\lambda^*,u^*)$ must satisfy:
\begin{align*}
d_t+w_t-\lambda^*-u_t^*=0
\end{align*}

Since we have $d_t\geq0$, by further manipulation we have:
\begin{align*}
w_t = \max\{\lambda^*-d_t,0\}
\end{align*}
Moreover, there is a unique choice of $\lambda^*$ such that the resulting weights $w_t$ sum up to $1$. 
This $w_t$ against $d_t$ relationship is shown in Figure \ref{fig:w_and_d}. This result is intuitive: as $d_t$ increases, the resulting weight we assign $w_t$ decreases, and if $d_t$ passes a certain threshold, it becomes large enough so that any reduction in variance it could provide is more than offset by its large bias, in which case we assign it a weight of $0$. 
 \begin{figure}[h]
 	\centering
 	\includegraphics[width=0.5\linewidth]{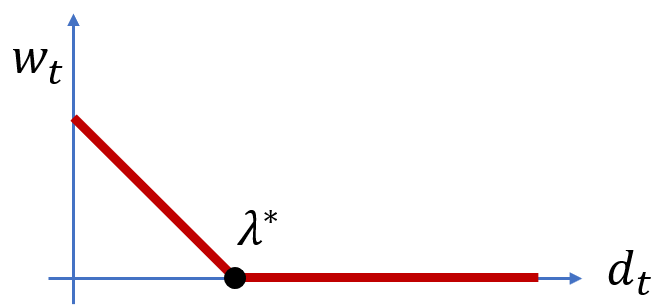}
 	\caption[]{$w_t-d_t$ relationship.}
 	\label{fig:w_and_d}
 \end{figure}

Our Temporal Weighting algorithm is in Algorithm \ref{alg:weighting}. During implementation, we adjust the relative importance of bias and variance by normalizing and scaling the graph distances (scaling factor 0.005 works well based on our empirical observation).
\begin{algorithm}
	\caption{Computing Temporal Weights $w_t$}
	\label{alg:weighting}
	\KwIn{distance $d_t$, with $t = 1,2,\cdots, N$}
	\KwOut{weights $w_t$, with $t = 1,2,\cdots, N$}
	Compute the unique $\lambda^*$ that satisfies:
	$$\sum_{t\in\{1,2,...N\}} \max\{\lambda^*-d_t,0\}=1$$
	
	Get weights $w_t$: 
	$$w_t = \max\{\lambda^*-d_t,0\}$$
\end{algorithm}

\subsection{Proposed Anomaly Detection Algorithm} \label{sec:anomaly}
Having obtained our weights $w_t$, the remaining step is to compute our anomaly score, as shown in Algorithm \ref{alg:anomaly}. 

We focus on 3 metrics from sensor data as indications of power system anomalies. These metrics were studied in \cite{hooi2018gridwatch} and found to be effective for detecting anomalies in power grid sensor data. In our setting, recall that for each sensor, we can obtain $\Delta s_i$ that contains changes of real and reactive power on the adjacent lines, over time. The 3 metrics are:
\bit
\item \emph{Edge anomaly metric}: $X_{edge,i}(t) = \max_{l\in{E_{adj}}}\Delta s_{i,l}$ which measures the maximum line flow change among lines connected to the sensor. Let $E_{adj}$ denote the set of lines connected to sensor $i$:
\item \emph{Average anomaly metric}: $X_{ave,i}(t) = \text{mean}\{\Delta s_{i,l}\ |\ l\in{E_{adj}}\},$ which measures the average line flow change on the lines connected to the sensor:
\item \emph{Diversion anomaly metric}: $X_{div,i}(t) = \text{std}\{\Delta s_{i,l}\ |\ l\in{E_{adj}}\},$ which measures the standard deviation of line flow change over all lines connected to the sensor:
\eit

Intuitively, for each metric, we want to estimate a model of its normal behavior. To do this, we compute the weighted median and interquartile range (IQR)\footnote{IQR is the difference between 1st and 3rd quartiles of the distribution, and is commonly used as a robust measure of spread.} of the detection metric, weighting the time points using our temporal weights $w_1, \cdots, w_t$. (Weighted) median and IQR are preferred choice of distribution parameters over the mean and variance for anomaly detection since they are robust measures of central tendency and statistical dispersion (i.e. they are less likely to be impacted by outliers)\cite{robust_statistics}. We can then estimate the anomalousness of the current time tick by computing the current value of a metric, then subtracting its weighted median and dividing by its IQR. The exact steps are given in Algorithm \ref{alg:anomaly}.

\begin{algorithm}
	\caption{Anomaly Detection (see Figure \ref{fig:toyexample_nonlocal})}
	\label{alg:anomaly}
	\KwIn{Temporal weights $w_t$; sensor data $s_i(t)$ with $t=1,2,...15$, $ i=1,2,...N_{sensor}$}
	\KwOut{anomaly score $A(15)$}
	\For{{$i \gets 1$ to $N_{sensor}$} }{
	
	
	
	{\bf Compute weighted median and IQR:}
	\begin{align*} 
	\mu_{edge}=\text{Weighted Median}\{X_{edge, i}(t)\ |\ t=1,...,14\}\\
	IQR_{edge}=\text{Weighted IQR}\{X_{edge, i}(t)\ |\ t=1,...,14\}
	\end{align*}
	\ \ weighted by $w_1,...w_{14}$ (similarly for $X_{ave}, X_{div}$). 
	
	{\bf Calculate sensor-wise anomaly score at t=15:} 
	\begin{align*}
	a_i(15) = \max_{metric\in\{edge, ave, div\}}\frac{X_{metric, i}(15)-\mu_{metric}}{IQR_{metric}}
	\end{align*}
	}

{\bf Calculate anomaly score for target time tick}, as the max score over sensors:
\begin{align*}
A(15) = \max_{i\in\{1,...,N_{sensor}\}}{ a_i(15)}
\end{align*}
\end{algorithm} 


\subsection{Proposed Locally Sensitive Distance Measure} \label{sec:local}

\subsubsection*{Motivation}
\begin{figure}[h]
	\centering
	\includegraphics[width=0.55\linewidth]{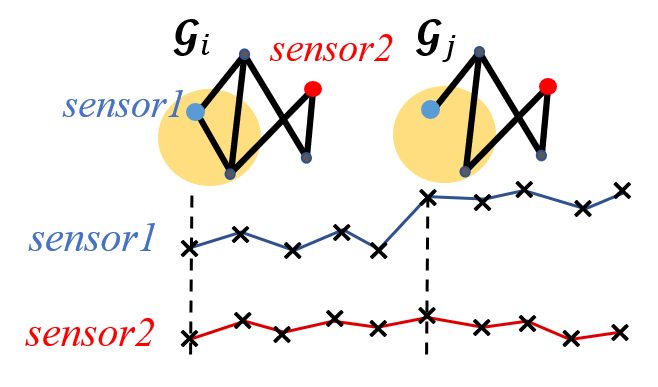}
	\caption[]{Simple motivating example: \methodls. The two graphs are very different within the yellow localized region. However, sensor 2 is far away from the yellow region and thus experiences no changes.} 
	\label{fig:toyexample_local}
\end{figure}
In the previous section, we computed a single distance value $D(\G_i, \G_j)$ between any pair of graphs. However, consider two graphs $\G_i$ and $\G_j$ in Figure \ref{fig:toyexample_local} that are very different due to a small yellow localized region (e.g. in a single building that underwent heavy renovation). Hence, $D(\G_i, \G_j)$ is large, indicating not to use data from $\G_i$ when we analyse a time tick under $\G_j$. However, from the perspective of a single sensor $s$ (sensor 2) far away from the localized region, this sensor may experience little or no changes in the power system's behavior, so that data from graph $\G_i$ may have a similar distribution as data from graph $\G_j$, and so for this sensor (sensor 2) we can still use data from $\G_i$ to improve anomaly detection performance. Hence, rather than computing a single distance $D(\G_i, \G_j)$, we compute a separate {\bf locally-sensitive} distance $D_s(\G_i, \G_j)$ specific to each sensor, which measures the amount of change between graphs $\G_i$ and $\G_j$ in the `local' region to sensor $s$. Clearly, the notion of `local regions' must be carefully defined: we will define them based on LODF, recalling that LODF measures how much changes on one edge affect each other edge.

\begin{figure}[h]
	\centering
	\includegraphics[width=0.5\linewidth]{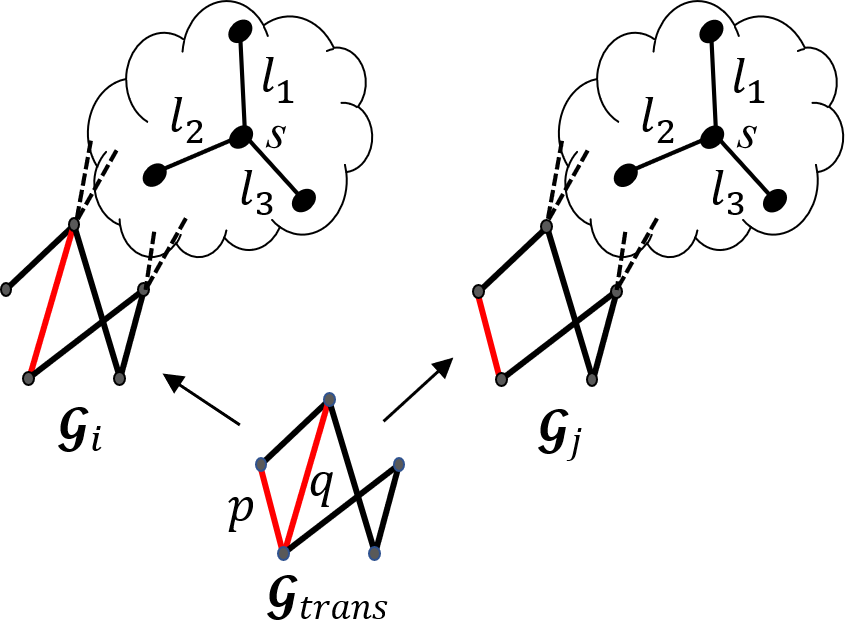}
	\caption[]{Local graph distance: the adjacent lines connected to each sensor s are considered.}
	\label{fig:localgraphdistance}
\end{figure}

Intuitively, the local distance between two graphs with respect to sensor $s$ is large if the changed edges can potentially cause large power change nearby the sensor. Hence, given two graphs $\G_i(\V(i),\E(i))$ and $\G_j(\V(j),\E(j))$ with their transition state $\G_{trans}$ and a sensor $s$ of interest, the local graph distance contribution $y_p$ of line $p$ with respect to sensor $s$ can be calculated by multiplying the whole-grid-wide contribution $x_p$ with a weighing factor $c_p^s$. This $c_p^s$ coefficient filters the power impact for sensor $s$ using the maximum power impact of line deletion on lines around this sensor:
\begin{align*}
c_p^s &= \max_{l\in\Exp_{sensor}(s)}|d_l^p|\\
y_p &= x_pc_p^s
\end{align*}
where $\Exp_{sensor}(s)$ denotes the set of edges around sensor $s$ (e.g. in Figure \ref{fig:localgraphdistance}, $\Exp_{sensor}(s)=\{l_1,l_2,l_3\}$), and $d_l^p$ denotes the LODF with $p$ as outage line and $l$ as observed line.)

Then, as before, the local graph distance with respect to sensor s is defined by summing up the local graph distance contributions of different line deletions from the graph:
$$D_s(\G_i, \G_j) = \sum_{p\in (\E(i)-\E(j)) \cup (\E(j)-\E(i))}y_p$$

\subsection{Statistical error analysis}\label{sec: statistical error analysis}

This section focuses on a quantitative analysis of the performance of our method, through statistical error analysis.

Let $T$ denote the width of time window for analysis, then for $\forall$ sensor $s$, the anomalousness of its observation $x_{T+1}$ is evaluated based on its previous data $x_1, x_2,...,x_T$. The anomaly detection method works by assigning weights $w_1, w_2, ..., w_T$ ($w_t\geq 0, \forall t, \sum_{t=1}^{T}w_t=1$) to all the previous observations and an alarm is created if $x_{T+1}$ deviates $\sum_{t=1}^{T}w_tx_t$ by a certain threshold. 

Here we investigate the properties of statistical error based on the following definitions and assumptions:

\begin{assumption}[Temporal independence]  
For any sensor $s$ and time $t$, its measured data $x_t$ is drawn from a Gaussian distribution $P(x_t)=N(\mu_t,\sigma^2)$ independently from other time ticks, where $\sigma^2$ accounts for all  uncertainties caused by measurement noise, load/generation variation, weather uncertainty, etc.
\label{assumption: temporal independence}
\end{assumption}

\begin{assumption}[identical distribution conditioned on topology]
Given a certain topology $G$ and $\forall$ sensor $s$, all data of $s$ under the same topology $G$ are drawn independently from the same distribution $P(x_t|G)=N(\mu_G,\sigma^2)$ (i.e., for any two time ticks $t_1,t_2$ with the same topology $G$, we have $\mu_{t_1}=\mu_{t_2}=\mu_G$.)
\label{assumption: identical distribution conditioned on G}
\end{assumption}

\begin{definition}[Optimal graph distance]
For any time-series data $x_1,x_2,...,x_T$ of a sensor $s$ and its latest observation $x_{T+1}$, $d_t$ denotes the graph distance between the graph at time $t$ and the graph at $T+1$, i.e., $d_t=D(G_t,G_{T+1}),d_t\geq 0,\forall t$. Then the optimal graph distance $d^*_{t}$ for $\forall t$ satisfies $|\mu_t-\mu_{T+1}|=|\mu_{G_t}-\mu_{G_{T+1}}|\propto d_t^*$, or equivalently, $\exists$ constant $c$ such that $|\mu_t-\mu_{T+1}|=c\cdot d_t^*$.
\label{def: optimal graph distance}
\end{definition}

We first demonstrate that the statistical error can be bounded:

\begin{restatable}[Error bound]{theorem}{errorbound}
Based on Assumption \ref{assumption: temporal independence},\ref{assumption: identical distribution conditioned on G} and Definition \ref{def: optimal graph distance}, the statistical error $\Exp_{x_1,x_2,...,x_T,x_{T+1}}[(\sum_{t=1}^{T}w_tx_t-x_{T+1})^2]$ with $w_t\geq0, \forall t$ and $\sum_{t=1}^{T}w_t=1$, satisfies:
$$\sigma^2\leq \Exp[(\sum_{t=1}^{T}w_tx_t-x_{T+1})^2]\leq (1+\max_t w_t)\sigma^2+c\max_t d_t^* $$
\label{theorem: error bound}
\end{restatable}

Upon obtaining the error bound, here are some intuitive explanations of the upper bound being dependent on $\max_t w_t$ and $\max_t d^*_t$:
\begin{itemize}
    \item $\max_t w_t$: large value for this term indicates that the estimation method depends heavily on a particular prior data point with weight $w_t$. This can lead to overfitting and as a result higher error (upper bound) due to high variance.
    \item $\max_t d^*_t$: large value for this term indicates that a prior data point from a very different distribution has been used for estimation, which can lead to higher error (upper bound) due to high bias.
\end{itemize}

\noindent Another question of interest to us is the properties in the limit of infinite data:
\begin{restatable}[Unbiased estimation under infinite data]{theorem}{errorinfdata}
In the limit of infinite data, the statistical error limits at the lower bound:
$$\Exp_{x_1,x_2,...,x_T,x_{T+1}}[(\sum_{t=1}^{T}w_tx_t-x_{T+1})^2]=\sigma^2$$
and an unbiased estimate of the true distribution $x_{T+1}\sim N(\mu_{T+1},\sigma^2)$ is obtainable using the previous samples, i.e.,
$$\Exp[\sum_{t=1}^{T}w_tx_t]=\mu_{T+1}$$
$$\Exp[\frac{1}{T-1}\sum_t(x_t-\frac{\sum_t x_t}{T})^2]=\sigma^2$$
\end{restatable}

\noindent Detailed proofs of the two theorems are included in Appendix.\ref{sec:proof of statistical analysis}.

\section{Experiments}\label{sec:experiments}
\label{sec:Results}
We design experiments to answer the following questions:
\begin{itemize}
    \item \textbf{Q1. Anomaly Detection Performance:} how accurate is the anomaly detection from our method compared to other ML baselines?
    \item \textbf{Q2. Scalability:} how do our algorithms scale with the graph size?
    \item\textbf{Q3. Practical Benefits:} how can our algorithm enhance the standard practices (SE BDD) in today's grid operator?
\end{itemize}

Our code and data are publicly available at \codeurl. 
Experiments were done on a 1.9 GHz Intel Core i7 laptop, 16 GB RAM running Microsoft Windows 10 Pro. 

\subsubsection*{\bf Case Data} We use 2 test cases: \datasmal is an accurate reconstruction of part of the European high voltage network, and \dataxlarge is a synthetic network that mimics the Texas high-voltage grid in the U.S. The \dataxlarge represents a similarly sized system as the PJM (the largest independent service operator (ISO) in the U.S.) grid, which contains around 25 to 30k buses~\cite{zimmerman2011matpower}. 



{\bf Selection of sensors and network observability:}  Due to the spatial impact of grid anomalies and the efficacy of anomaly metrics, full observability\cite{SE_observability}\cite{grid_observability_analysis} and optimal sensor placement for observability\cite{OPP_observability} are not necessary for Dynwatch to perform. However, access to more sensors as inputs alongside optimal sensor selection\cite{hooi2018gridwatch} can improve the detection performance and help with localization of anomaly. To be conservative, for these experiments, we select random subsets of sensors (of varying sizes) as input. The good performance even with randomly selected sensor measurements validates the effectiveness of our method in selecting relevant time frames from historical data.

{\bf Threshold tuning:} For a fair comparison of different methods, our experiment section, without using any threshold, compares the top K anomalies scored by each algorithm, where K is the number of anomalies simulated. However, in practical use, a detection threshold is required for the algorithm to identify an anomaly. A proper threshold can be either a fixed threshold from an empirical value or domain-knowledge or a learned threshold to facilitate optimal decision making. In particular, optimal threshold tuning needs to take {\bf class imbalance and asymmetric error} into account. Since only a minority of instances are expected to be abnormal, there is an unbalanced nature of data. Moreover, as grid applications are safety-critical, mislabeling an anomaly as normal, i.e., false negatives (FN), could cause fatal consequences, while false positives (FP) which cause loss of ‘fidelity’ are less serious. Proper techniques for tuning a threshold {\bf offline} include: 
\begin{enumerate}
    \item Calculating the evaluation metric (e.g. F-measure which quantifies the balance of precision and recall) for each threshold to select the one that maximizes the metric. 
    \item Plotting the ROC curve or precision-recall curve to select the threshold that gives the optimal balance.
    \item A cost approach that, when the cost of FP, FN, TP, TN are available, minimizes the average overall cost of a diagnostic test, yet domain-specific knowledge is needed for reasonable quantification of the costs.
\end{enumerate}

\subsection{Q1. Anomaly Detection Accuracy}

\begin{figure*}[htp]
	\centering
	\includegraphics[width=0.7\linewidth]{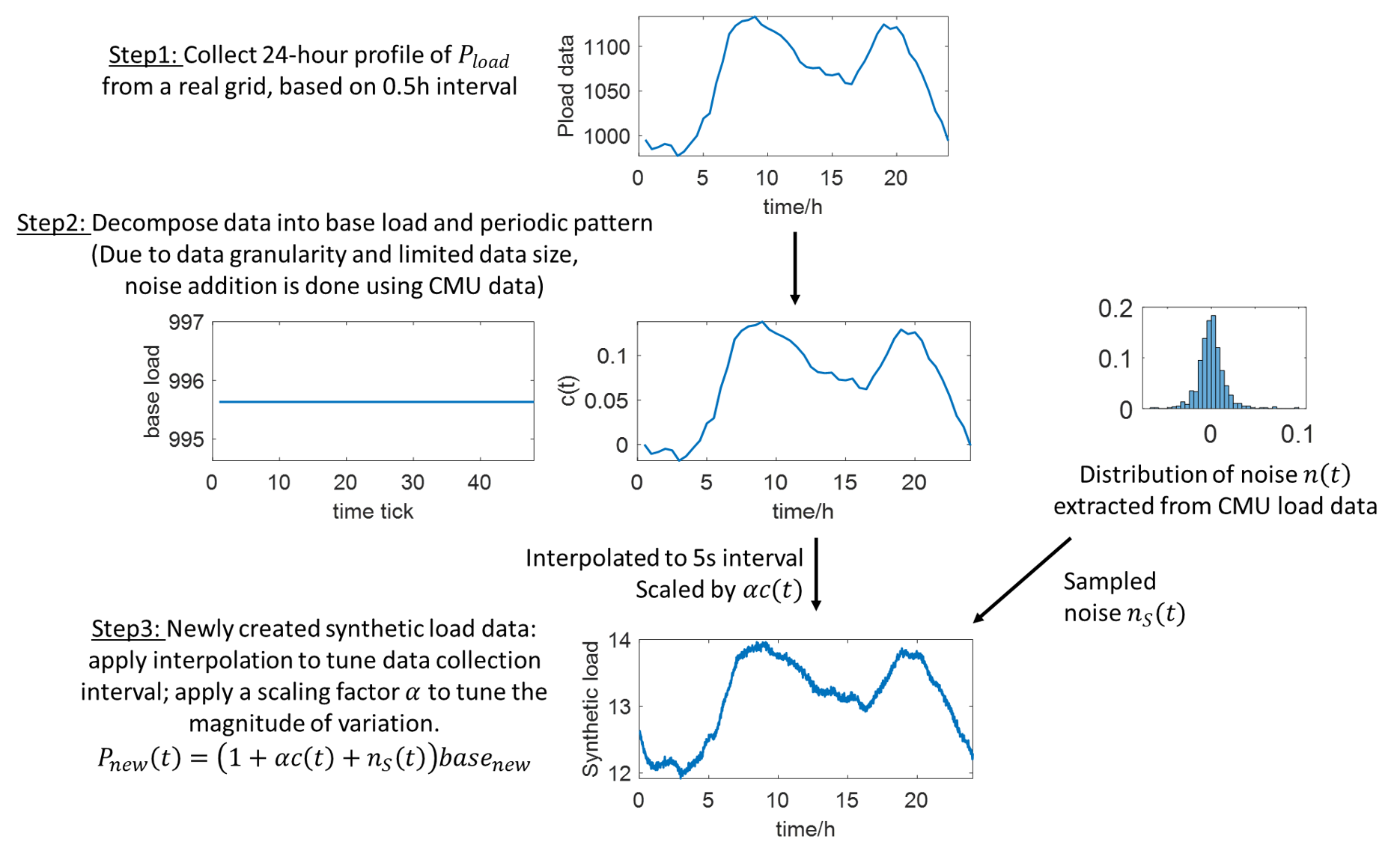}
	\caption[]{Methodology for construction of synthetic data from a real utility-provided load data (See Section \ref{sec: realistic patterns}). The generated data is used for evaluating performance in Section \ref{sec:experiments} where the 1200 time-ticks are used to model the time-series data of grid operation over 1h 40min.}
	\label{fig:synthetic data TVA}
\end{figure*}

In this section, we compare \method against baseline anomaly detection approaches, while varying the number of sensors in the grid. 

\subsubsection*{\bf Experimental Settings} 
Starting with a particular test case as a base graph $G$, we first create $20$ different topology scenarios where each of them deactivates a randomly chosen branch in the base graph. These subsequent $20$ network topologies represent the dynamic grid with topology changes due to operation and control. Then for each topology scenario, we use MatPower~\cite{zimmerman2011matpower}, a standard power grid simulator, to generate $60$ sets of synthetic measurements based on the load characteristics described in the following paragraph. As a result, the multivariate time series with $20\times60=1200$ time ticks mimics the real-world data setting where sensors receive measurements at each time tick $t$, and the grid topology changes every $60$ time ticks. Finally, we sample $50$ random ticks out of $1200$ as times when anomalies occur. Each of these anomalies is added by randomly deleting an edge on the corresponding topology.


Following \cite{hooi2018gridwatch}, to generate an input time series of loads (i.e. real and reactive power at each node), we use the patterns estimated from two real datasets:
\begin{itemize}
    \item Carnegie Mellon University (CMU) campus load data recorded for $20$ days from July 29 to August 17, 2016;
    \item Utility-provided 24 hour dataset of a real U.S. grid. See Section \ref{sec: realistic patterns} for dataset description and statistical findings related to it.
\end{itemize}
to create synthetic time-series load based on a 5s interval, with the magnitude of daily load variation scaled to a predefined level, and with added Gaussian noise sampled from the extracted noise distribution~\cite{song2017powercast}. The detailed data generation process is shown in Figure \ref{fig:synthetic data TVA}.

Given this input, each algorithm then returns a ranking of the anomalies. We evaluate this using standard metrics, AUC\footnote{AUC is the probability of correct ranking of a random “positive”-“negative” pair.} (area under the ROC curve) and F-measure\footnote{F-measure is a trade-off between precision and recall.} ($\frac{2\cdot \text{precision} \cdot \text{recall}}{\text{precision} + \text{recall}}$), the latter computed on the top $50$ anomalies output by each algorithm. 

\subsubsection*{\bf Baselines} Dynamic graph anomaly detection approaches \cite{akoglu2010oddball,chen2012community,araujo2014com2,shah2015timecrunch} cannot be used as they consider graph structure only, but not sensor data. \cite{mongiovi2013netspot} allows sensor data but requires graphs with fully observed edge weights, which is inapplicable as detecting failed power lines with all sensors present reduces to checking if any edge has current equal to $0$. Hence, instead, we compare \method to GridWatch~\cite{hooi2018gridwatch}, an anomaly detection approach for sensors on a static graph, and the following multidimensional time-series based anomaly detection methods: Isolation Forests~\cite{liu2008isolation}, Vector Autoregression (VAR)~\cite{hamilton1994time}, Local Outlier Factor (LOF)~\cite{breunig2000lof}, and Parzen Window~\cite{parzen1962estimation}. Each uses the currents and voltages at the given sensors as features. For VAR, the norms of the residuals are used as anomaly scores; the remaining methods return anomaly scores directly.
For Isolation Forests, we use $100$ trees (following the defaults in scikit-learn\cite{scikit-learn}). For VAR, following standard practice, we select the order by maximizing AIC. For LOF we use $20$ neighbors (following the default in scikit-learn), and we use $20$ neighbors for Parzen Window.

As shown in Figure \ref{fig:result_all}, \method clearly outperforms the baselines on both metrics, having an F-measure of $>20\%$ higher than the best baseline. 
\begin{figure}[h]
	\centering
	\includegraphics[width=0.75\linewidth]{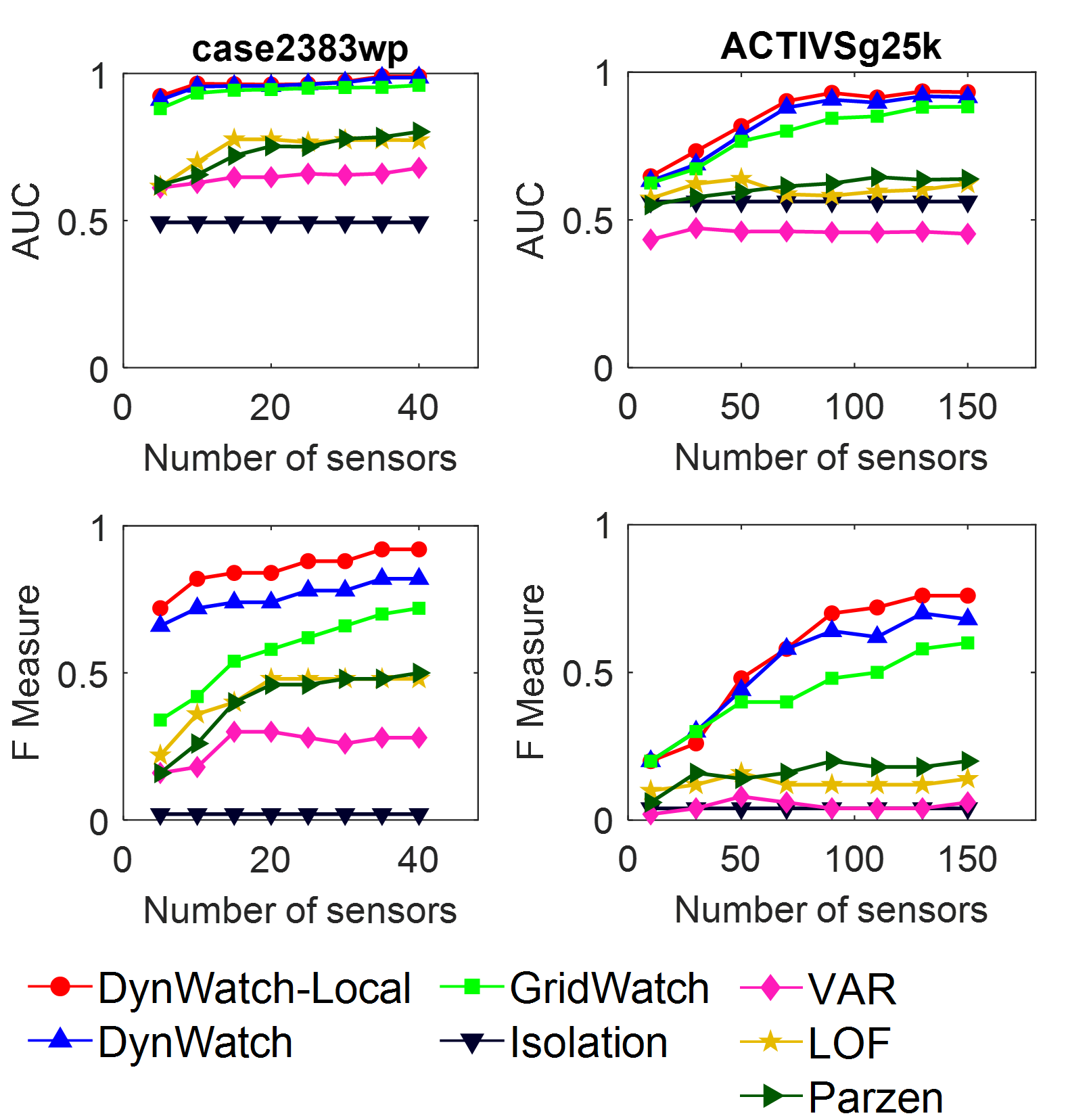}
	\caption[]{Experiment results by AUC and F-measure.}
	\label{fig:result_all}
\end{figure}
\subsection{Q2. Scalability}
In this subsection, we seek to analyze the scalability of our \method and \methodls. In reality, PJM, the largest ISO in the U.S., runs ACSE on a 28k bus model, performed every 1 min\cite{PJM-manual-12}, thus any anomaly detection algorithm that takes significantly less than 1 min may provide valuable information to prevent wrong control decisions in real-time. The following results demonstrate the proposed method's capability to achieve this.


Here, 
we generate test cases of different sizes by starting with the \datasmal case and duplicating it $3,4,5,\cdots,12$ times. After each duplication, edges are added to connect each node with its counterpart in the last duplication, so that the whole grid is connected. Then for each testcase, we generate 20 dynamic grids and sensor data with 1 randomly chosen sensor and 1200 time ticks, following the same settings as the previous sub-section. Finally, we measure 
\begin{figure}[h]
	\centering
	\includegraphics[width=1\linewidth]{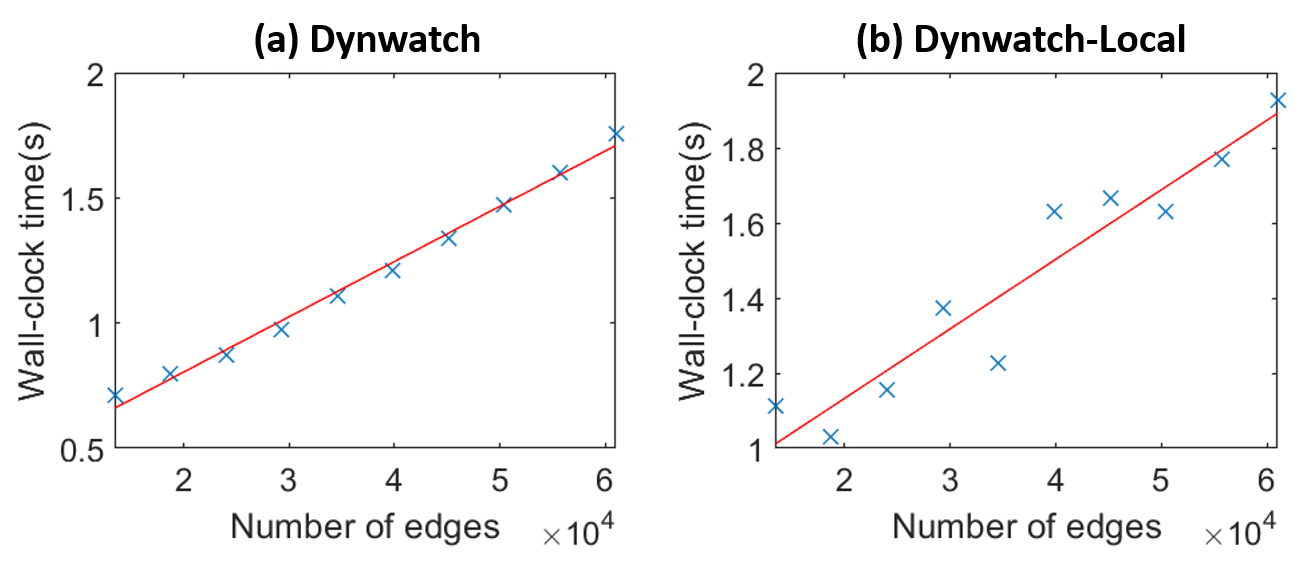}
	\caption[]{\label{fig:scability_all_new} Our algorithms scale linearly: wall-clock time of (a) \method; and (b) \methodls  against number of edges, when detecting on all 1200 time ticks. The red lines are best-fit regression lines. }
\end{figure}

Figure \ref{fig:scability_all_new} shows that our method is fast: even on a large case with 60k+ branches, both methods took less than 2s to apply anomaly detection on all 1200 time ticks of the sensor, corresponding to an average of less than 1.7ms per time tick per sensor. The \dataxlarge (similar to the largest real-time network in the US) has 32k+ branches, and thus run-time of anomaly detection at each time $t$ will be significantly less than 1 min. The figure also shows that our methods scale close to linearly with the grid size. 



\subsection{Q3. Practical Benefits}\label{sec:experiment on FDIA}

In this section, we explore how the proposed Dynwatch algorithm can improve the performance of the standard residual-based ACSE bad-data detection (BDD) method, by testing a type of grid-specific anomaly that SE BDD is known to fail against False Data Injection Attack (FDIA).

False Data Injection Attack (FDIA)\cite{fdia-review} is a cyber-attack in which attackers manipulate the value of measurements according to the grid physical model such that the SE outputs incorrect grid estimates while ensuring that its residual does not change by much (ideally remains unchanged). 

In this experiment, we construct an attack on a 14 bus network to mislead the operator into thinking that the load reduces by 20\%. For any anomalous time tick $t$, this is implemented by simulating power flow with the reduced load and generating measurements based on that. 

The time-series measurement data and a comparison of anomaly scores are shown in Figure \ref{fig:FDIA example}. Results show that the ACSE residual, which is metric for BDD reduced in value when anomalies occur (see the residual decrease in Figure \ref{fig:FDIA example} during anomalous operation shown by the dotted red line), implying that the standard SE BDD, along with any other residual-based method, will not be able to detect this coordinated attack.

In addition to standard ACSE BDD, we also implemented the auto-regression (VAR) method to detect grid anomalies. As the VAR algorithm does not consider dynamic graph properties of the power grid, it tends to create alarms on all abrupt measurement changes. This will easily lead to false positives since regular topological changes also result in sudden temporal change. This can be seen in the Figure \ref{fig:FDIA example} wherein during regular topology changes (shown in black dotted line) sudden spikes in VAR anomaly score are observed.

In comparison, our proposed method is able to detect all anomalies without False positives (FP). This indicates that the proposed algorithm is more likely to detect anomalies due to complex attack scenarios while being able to reduce the occurrences of false positives. 

\begin{figure*}[htp]
\centering
\includegraphics[width=0.75\linewidth]{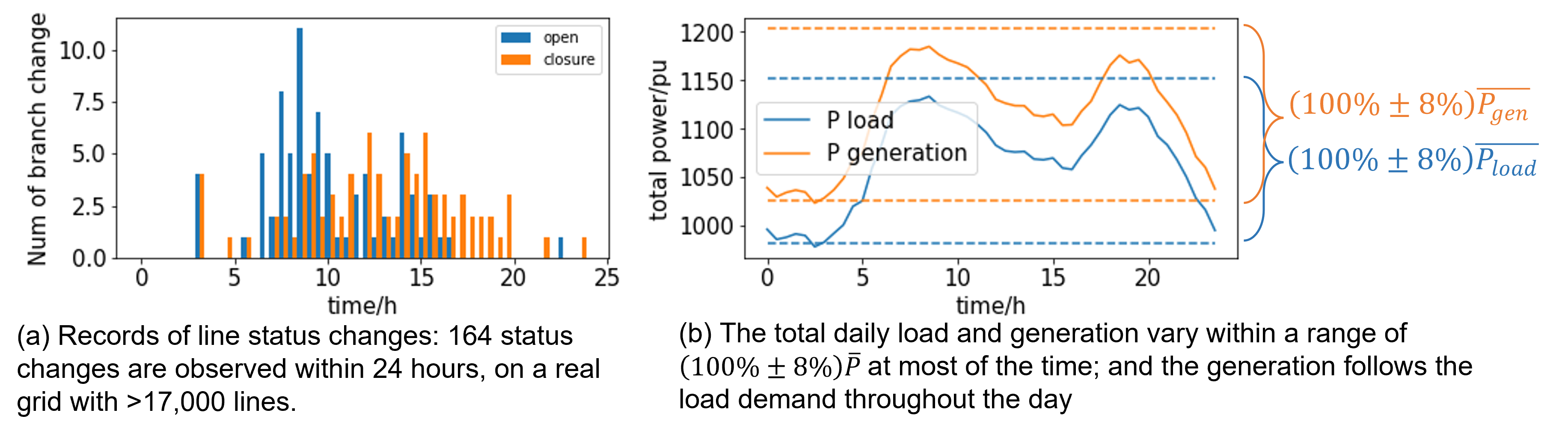}
\caption[]{Change of line status, total load and generation on for a real-world load dataset.}
\label{fig: topology and total power change}    
\end{figure*}

\begin{figure*}[htp]
\centering
\includegraphics[width=0.8\linewidth]{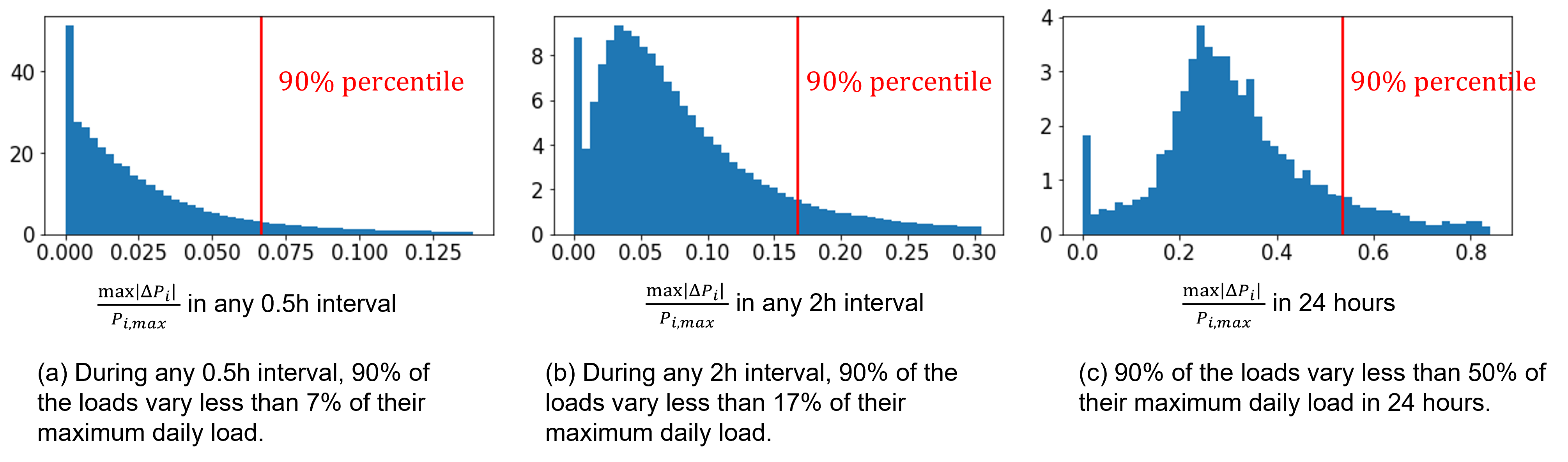}
\caption[]{Distribution of individual load variations for real-world load dataset.}
\label{fig: distribution of load changes}
\end{figure*}

\section{Further Considerations}\label{sec: discussion}

Some special conditions of interest are related to the application of our proposed method. They are discussed below.

\subsection{Realistic grid patterns}\label{sec: realistic patterns}

Here we explore the realistic grid pattern, using a \textbf{utility-provided real-world dataset} (for reasons of confidentiality, we cannot make this dataset public). 

\textbf{Dataset description-} 24 hours worth of ACSE output data from a real utility in the Eastern Interconnect of the U.S. The dataset contains all operational data of the grid based on a 0.5-hour interval. 

Here we document the statistical findings with the following notable observations:

\begin{itemize}
    \item Figure \ref{fig: topology and total power change}(a): Topology changes frequently in reality, thus our proposed method, which considers the impact of topology change, provides realistic values in handling time-series sensor data. 
    \item Figure \ref{fig: topology and total power change}(b): Real part of daily loads and generations change smoothly within a range of $(100\pm 8)\%$ times their average values, with the generation following the load demand throughout the day. 
    \item Figure \ref{fig: distribution of load changes}: $90\%$ of the time, an individual real part of load changes less than $10\%$ of its maximum daily value within a 30-min interval; less than $17\%$ within a 2-hour interval; and less than $50\%$ within 24 hours.
\end{itemize}

\noindent We observed similar behavior for reactive power in the system as well.

\subsection{Non-synchronization of measurements}
In the worst possible case of measurement non-synchronization, the measurements and status data can be collected at such a different time that the measurements and reference topology are inconsistent with each other. If this happens, our method should detect it as an anomaly, since they are equivalent to topological or measurement errors.

However, the impact of this event is low. In case “uncorrected” unsynchronized measurements or topology exists, then that is an error that would disrupt the grid operator’s state estimation. In practice, the grid has some mechanisms to prevent this from happening: 

\begin{enumerate}
    \item Both status data and analog measurements are sampled very frequently: status data are updated upon change of status in PJM and every 4s in ISO-NE; analog measurements are updated every 10s in PJM. Considering that the frequency of topology change is low, the high acquisition frequency is likely to guarantee all information is up to date.
    \item Topology estimation (TE) can check the consistency between switch status and analog measurements: e.g., if the switch is observed as ‘open’ but the current measurement on it is non-zero, then this status data is wrong and should be corrected. This processing is within the TE algorithms in the control room.
\end{enumerate}
 
Therefore, it is likely that we can still obtain a good reference topology not affected by the non-synchronization.

\subsection{Abrupt changes in aggregated load or generation}
Considering that: 1) the load changes slowly following a daily cycle; 2) conventional generators are re-dispatched every 15 min (less frequent than data acquisition we have sufficient time ticks before each re-dispatch) and each re-dispatch is limited by ramp rate; 3) renewable forecasts are improving and the auxiliary services like batteries are making them more stable and “firm” sources of energy, it’s still reasonable to assume that for any time $t$, its recent previous measurements can provide meaningful information.

\textit{What if sudden load change happens? } At transmission grid level the load is aggregated at high-voltage nodes because of which abrupt change in load is uncommon. However, there exist rare events like load shedding or an unusual re-dispatch where there might be a sudden redistribution of power in the grid in a very short interval. In case of these rare events, the proposed method may detect the abrupt change as an anomaly. In reality, these events may not be anomalous and will be misclassified as anomalies. During such instances, we will rely on the operator to clear these false positives in their decision-making process where additional trustworthy information is available. This is reasonable since these events are rare and the operators are well aware of the situation of active ongoing load-shedding in a region. In case that the abrupt load change is unexpected to the operator, an indicator for anomalous behavior might be a rather useful one.

\section{Conclusion}
\label{sec:Conclusion}
In this paper, we proposed \method, an online algorithm that accurately detects anomalies using sensor data on a changing graph (grid). \method applies a similarity-based approach to measure how much the graph changes over time, with which we assign greater weight to previous graphs which are similar to the current graph. We use a domain-aware graph similarity measure based on Line Outage Distribution Factors (LODF), which exploit physics-based modeling of how changes in one line affect other lines in the graph. 

By plugging in different graph similarity measures, our approach could be applied to other domains. Hence, future work could study how sensitive various detectors are for detecting anomalies in graph-based sensor data from different domains.

\ben
\item {\bf Problem Formulation and Algorithm:} we propose a novel and practical problem formulation, of anomaly detection using sensors on a changing graph. For this problem, we propose a graph-similarity based approach, and a domain-aware similarity measure based on Line Outage Distribution Factors (LODF). 
\item {\bf Effectiveness:} Our \method algorithm outperforms existing approaches in accuracy by 20$\%$ or more (F-measure) in experiments. 
\item {\bf Scalability:} \method is fast, taking 1.7ms on average per time tick per sensor on a 60k edge graph, and scaling linearly in the size of the graph. 
\een
\textbf{Reproducibility:} our code and data are publicly available at \codeurl.


\section*{Acknowledgment}
Work in this paper is supported in part by C3.ai Inc. and Microsoft Corporation.
\bibliographystyle{IEEEtran}
\bibliography{refbib}

\appendix
\section{Appendix}

\subsection{Comparison of graph distance measures}\label{apdx: compare distances}
To quantitatively justify the effectiveness of our proposed graph distance, we compared the proposed distance with other traditional measures applicable to power grids:
\begin{itemize}
    \item Simple GED\cite{GED_MCS_shoubridge2002}: the distance between two graphs is equal to the number of changed edges.
    \item Variant of GED with line admittance used as weights assigned to the changed edges. Admittance is used here because the larger the admittance, the more likely the edge has large power flows on it, meaning it is important to the grid.
    \item Equivalent conductance-based measure: the distance between two graphs is equal to the sum of the equivalent conductance of all changed edges. Equivalent conductance is able to take more consideration of the system-wise impact of each edge.
\end{itemize}

Result in Figure \ref{fig:compare different distances} shows our proposed measure outperforms the baselines above. Here the time series data is generated using the pattern from the utility-provided data set, following the process described in Figure \ref{fig:synthetic data TVA}.
\begin{figure}[h]
	\centering
	\includegraphics[width=0.95\linewidth]{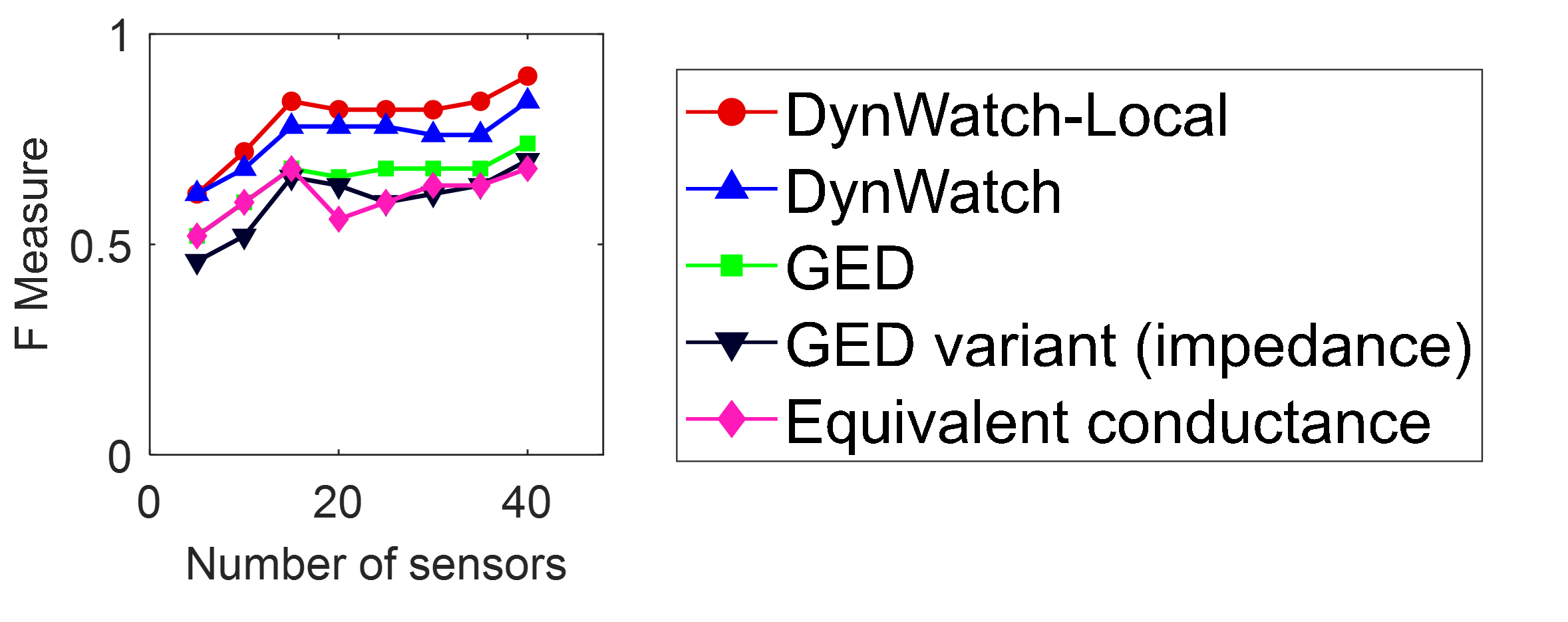}
	\caption[]{Result of F-measure on case2383wp, with 40 sensors installed: the proposed LODF-based graph distance outperforms other distance measures.}
	\label{fig:compare different distances}
\end{figure}

\subsection{Proofs of statistical error analysis}\label{sec:proof of statistical analysis}

Here we will prove the theorems demonstrated in Section \ref{sec: statistical error analysis}.

\errorbound*
\begin{proof}
\begin{align}
    &\Exp_{x_1,x_2,...x_T}[(\sum_{t=1}^{T}w_tx_t-x_{T+1})^2] \\
    =& \Exp[(\sum_{t=1}^{T}w_tx_t-\mu_{T+1}+\mu_{T+1}-x_{T+1})^2]\\
    =&\Exp[(\sum_{t=1}^{T}w_tx_t-\mu_{T+1})^2] + \Exp[(\mu_{T+1}-x_{T+1})^2] \\
    &+ 2\Exp[(\sum_{t=1}^{T}w_tx_t-\mu_{T+1})(\mu_{T+1}-x_{T+1})]
\end{align}
Based on Assumption \ref{assumption: temporal independence}, we have $\Exp[\mu_{T+1}-x_{T+1}]=0$, and thus
\begin{align*}
    &\Exp[(\sum_{t=1}^{T}w_tx_t-\mu_{T+1})(\mu_{T+1}-x_{T+1})\\
    =&\Exp[\sum_{t=1}^{T}w_tx_t(\mu_{T+1}-x_{T+1})]-\mu_{T+1}^2+\mu_{T+1}\Exp[x_{T+1}]\\
    =&\Exp[\sum_{t=1}^{T}w_tx_t]\Exp[\mu_{T+1}-x_{T+1}]-\mu_{T+1}^2+\mu_{T+1}^2\\
    =&0
\end{align*}
Therefore we have
\begin{align}
    &\Exp_{x_1,x_2,...x_T,x_{T+1}}[(\sum_{t=1}^{T}w_tx_t-x_{T+1})^2] \\
    =&\Exp[(\sum_{t=1}^{T}w_tx_t-\mu_{T+1})^2] + \Exp[(\mu_{T+1}-x_{T+1})^2] \\
    =&\Exp[(\sum_{t=1}^{T}w_tx_t-\sum_{t=1}^{T}w_t\mu_t+\sum_{t=1}^{T}w_t\mu_t-\mu_{T+1})^2] \\
    &+ \Exp[(\mu_{T+1}-x_{T+1})^2] \\
    =&\Exp[(\sum_{t=1}^{T}w_tx_t-\sum_{t=1}^{T}w_t\mu_t)^2] +\Exp[(\sum_{t=1}^{T}w_t\mu_t-\mu_{T+1})^2]\\
    &+2\Exp[(\sum_{t=1}^{T}w_tx_t-\sum_{t=1}^{T}w_t\mu_t)(\sum_{t=1}^{T}w_t\mu_t-\mu_{T+1})]\\
    & +\Exp[(\mu_{T+1}-x_{T+1})^2] 
\end{align}
Similarly based on Assumption \ref{assumption: temporal independence}, it is easy to show that
\begin{align*}
    \Exp[(\sum_{t=1}^{T}w_tx_t-\sum_{t=1}^{T}w_t\mu_t)(\sum_{t=1}^{T}w_t\mu_t-\mu_{T+1})]=0
\end{align*}
Thus we have
\begin{align}
    &\Exp_{x_1,x_2,...x_T,x_{T+1}}[(\sum_{t=1}^{T}w_tx_t-x_{T+1})^2]\\
    =&\underbrace{\Exp[(\sum_{t=1}^{T}w_tx_t-\sum_{t=1}^{T}w_t\mu_t)^2]}_\text{Variance}+
    \underbrace{E[(\sum_{t=1}^{T}w_t\mu_t-\mu_{T+1})^2]}_\text{Bias$^2$}\\
    &+\underbrace{\Exp[(\mu_{T+1}-x_{T+1})^2] }_\text{irreducible error}
\end{align}
Based on Assumption \ref{assumption: temporal independence} and $w_t\geq0$ for $\forall t, \sum_{t=1}^{T}w_t=1$, the variance term can be upper bounded by:
\begin{align*}
    \Exp[(\sum_{t=1}^{T}w_tx_t-\sum_{t=1}^{T}w_t\mu_t)^2]&=\sum_{t=1}^{T}w_t^2\Exp[(x_t-\mu_t)^2]\\
    &\leq (\max_tw_t)\sigma^2
\end{align*}
Further making use of Assumption \ref{assumption: identical distribution conditioned on G}, it is easy to get an upper bound for the bias$^2$ term:
\begin{align*}
    E[(\sum_{t=1}^{T}w_t\mu_t-\mu_{T+1})^2]&=\Exp[(\sum_{t=1}^{T}w_t|\mu_t-\mu_{T+1}|)^2]\\
    &=\sum_{t=1}^{T}w_tcd^*_t\\
    &\leq c\max_td^*_t
\end{align*}
Finally, as $\Exp[(\mu_{T+1}-x_{T+1})^2]=\sigma^2$ based on the assumption that $x_{T+1}\sim N(\mu_{T+1},\sigma^2)$, we are able to \textbf{upper bound} the statistical error as:
\begin{align}
    &\Exp_{x_1,x_2,...x_T,x_{T+1}}[(\sum_{t=1}^{T}w_tx_t-x_{T+1})^2]\\
    &\leq (\max_t w_t)\sigma^2+\max_t cd_t^*+\sigma^2\\
    &= (1+\max_t w_t)\sigma^2+c\max_td_t^* 
\end{align}
Meanwhile the \textbf{lower bound} is also obvious:
\begin{align}
     \Exp_{x_1,x_2,...x_T,x_{T+1}}[(\sum_{t=1}^{T}w_tx_t-x_{T+1})^2]\geq \sigma^2
\end{align}
\end{proof}

\errorinfdata*
\begin{proof}
In the limit of infinite data, there exist infinite data with the same topology as $x_{T+1}$, thus it is possible to find the time series data such that $x_1,x_2,...,x_T$ are drawn independently from the same distribution, s.t, $x_t\sim N(\mu_{T+1},\sigma^2), \forall t$ and $T\longrightarrow \infty$ . 

Thus we have $w_t=\frac{1}{T}\longrightarrow 0, d^*_t=0$ for $\forall t\in{0,1,...,T}$ and it is easy to show from Theorem \ref{theorem: error bound} that:
$$\sigma^2\leq \Exp_{x_1,x_2,...,x_T,x_{T+1}}[(\sum_{t=1}^{T}w_tx_t-x_{T+1})^2]\leq  \sigma^2$$
Thus,
$$\Exp_{x_1,x_2,...,x_T,x_{T+1}}[(\sum_{t=1}^{T}w_tx_t-x_{T+1})^2]= \sigma^2$$

Also we have,
$$\Exp[\sum_{t=1}^{T}w_tx_t]=\sum_{t=1}^{T}w_t\Exp[x_t]=\mu_{T+1}$$

And based on Bessel's correction, it is easy to get
$$\Exp[\frac{1}{T-1}\sum_t(x_t-\frac{\sum_t x_t}{T})^2]=\sigma^2$$
\end{proof}

\begin{IEEEbiography}
[{\includegraphics[width=1in,height=1.25in,clip,keepaspectratio]{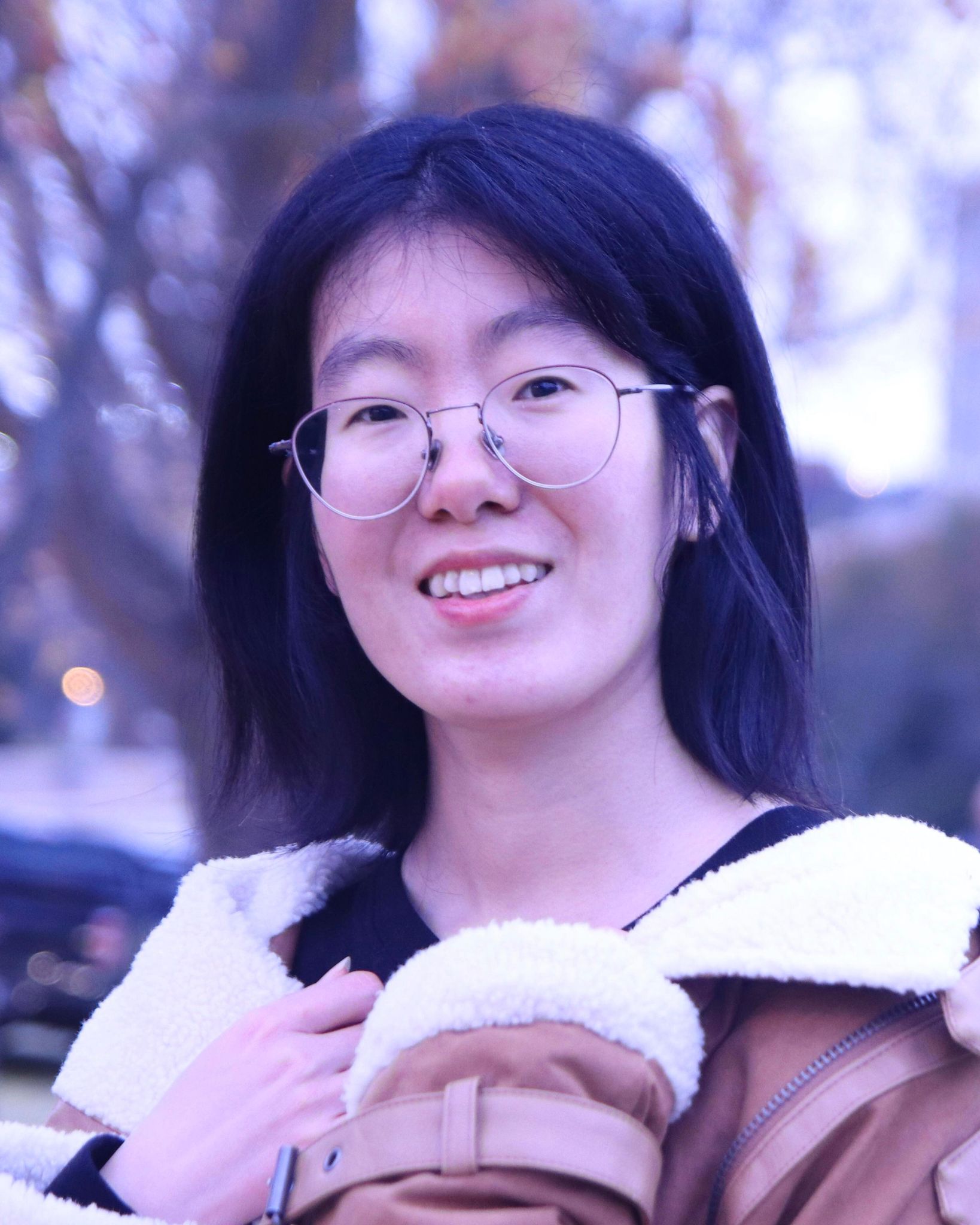}}] 
{Shimiao Li} received the B.E. degree in electrical engineering
from Tianjin University, China, in 2018. She is currently a Ph.D. candidate in the Department of Electrical and Computer Engineering (ECE) at Carnegie
Mellon University, where she is advised by Lawrence Pileggi. Her research interests include analytical methods for grid operation and optimization, as well as physics-informed machine learning for real-world grid cyber-security, reliability and optimization appliations.
\end{IEEEbiography}

\begin{IEEEbiography}
[{\includegraphics[width=1in,height=1.25in,clip,keepaspectratio]{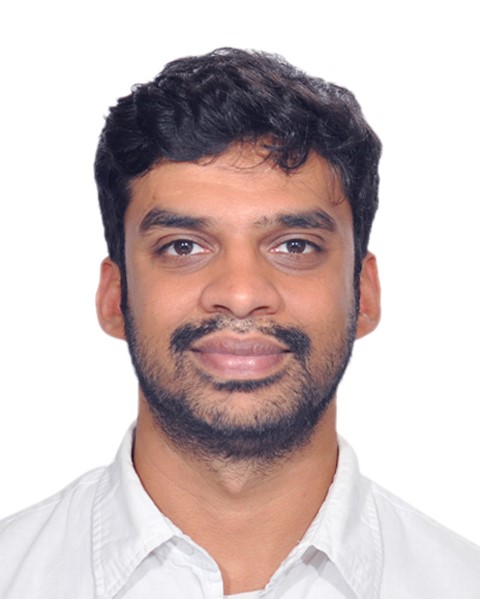}}] 
{Amritanshu Pandey} was born in Jabalpur, India. He received the M.Sc. and Ph.D. degree in Electrical Engineering in 2012 and 2019, respectively, from Carnegie Mellon University, Pittsburgh, PA, USA. He is currently a Special Faculty in the Electrical and Computer Engineering department at Carnegie Mellon University. He has previously worked at Pearl Street Technologies, Inc. as Senior Research Scientist and as an Engineer at MPR Associates, Inc. He develops first principles-based physics driven and ML-based data driven methods and models for energy systems that are expressive (i.e., capture sufficient physical behavior accurately) and efficient in terms of speed and scalability.
\end{IEEEbiography}

\begin{IEEEbiography}
[{\includegraphics[width=1in,height=1.25in,clip,keepaspectratio]{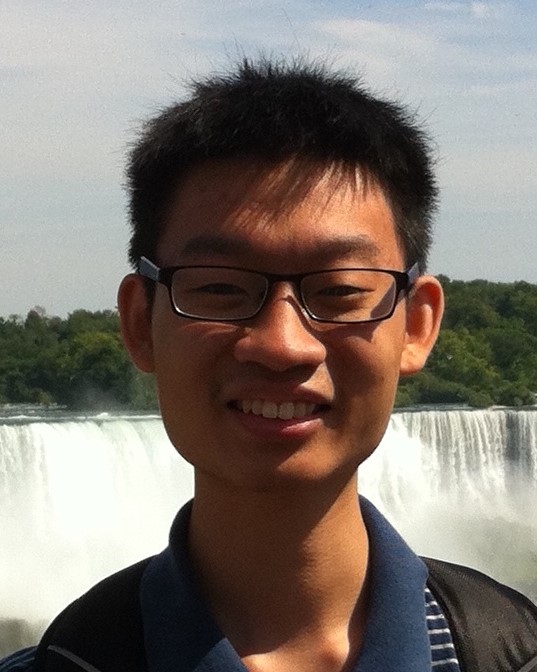}}] 
{Bryan Hooi} received his Ph.D. in Machine Learning from Carnegie Mellon University, where he was advised by Christos Faloutsos. He received the M.S. in Computer Science and the B.S. in Mathematics from Stanford University. He is currently an Assistant Professor in the School of Computing and the Institute of Data Science in National University of Singapore. His research interests include scalable machine learning, deep learning, graph algorithms, anomaly detection, and biomedical applications of AI.
\end{IEEEbiography}

\begin{IEEEbiography}
[{\includegraphics[width=1in,height=1.25in,clip,keepaspectratio]{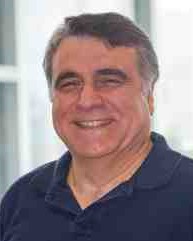}}] 
{Christos Faloutsos} is a Professor at Carnegie Mellon University and an Amazon Scholar. He is the recipient of the Fredkin Professorship in Artificial Intelligence (2020); he has received the Presidential Young Investigator Award by
the National Science Foundation (1989), the Research Contributions Award in ICDM 2006, 
the SIGKDD Innovations Award (2010), the PAKDD Distinguished Contributions Award (2018), 29 ``best paper'' awards (including 8 ``test of time'' awards),
and four teaching awards. Eight of his advisees or co-advisees have attracted KDD or
SCS dissertation awards.  He is an ACM Fellow,
he has served as a member of the executive committee of SIGKDD;
he has published over 400 refereed articles, 17 book chapters
and three monographs.  He holds 8 patents (and several more are pending), and he has given over 50 tutorials and over 25 invited distinguished lectures.
His research interests include large-scale data mining with emphasis on graphs and time sequences; anomaly detection, tensors, and fractals.
\end{IEEEbiography}

\begin{IEEEbiography}
[{\includegraphics[width=1in,height=1.25in,clip,keepaspectratio]{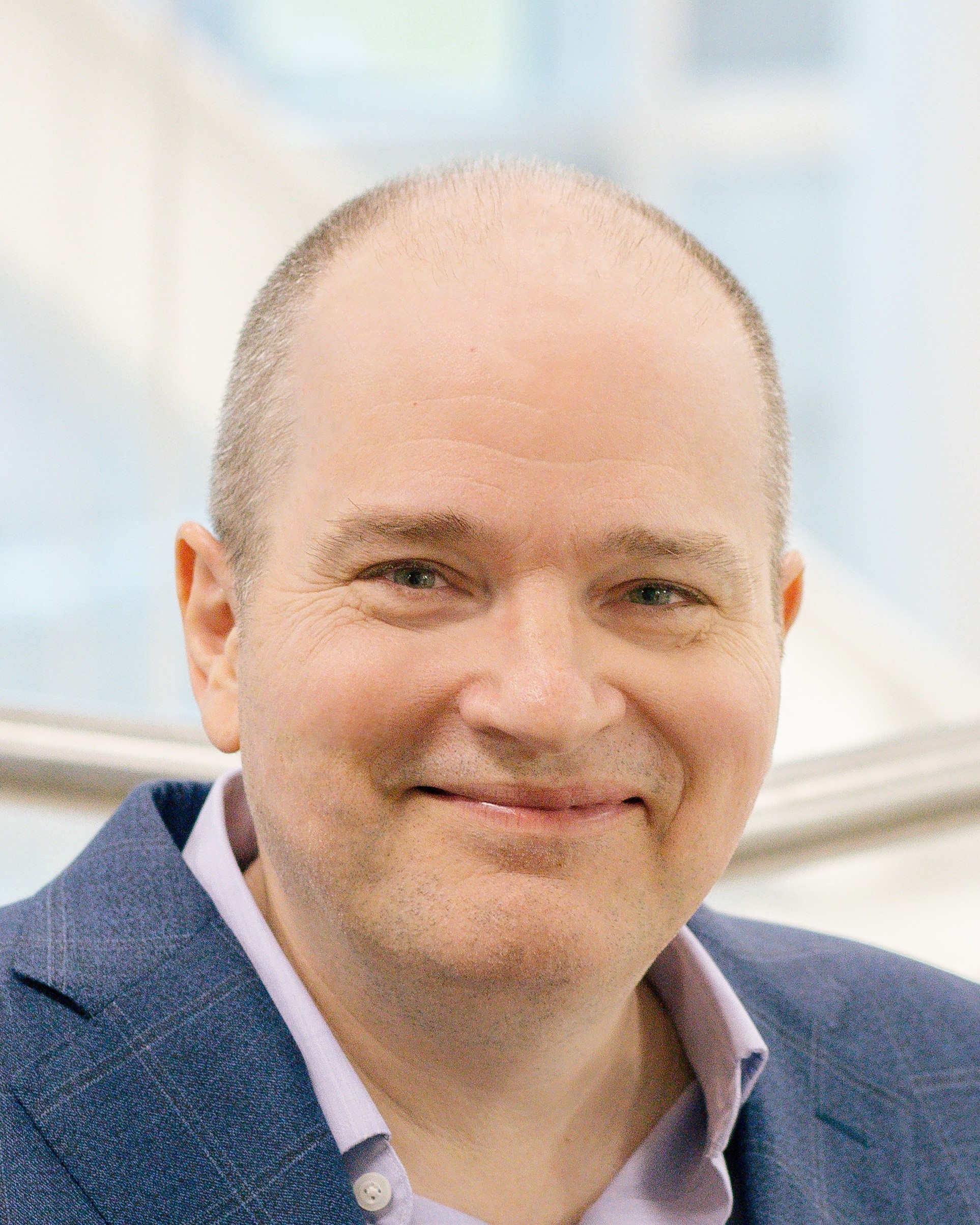}}] 
{Lawrence Pileggi} is the Tanoto professor and Head of electrical and computer engineering at Carnegie Mellon University, and has previously held positions at Westinghouse Research and Development and the University of Texas at Austin. He received his Ph.D. in Electrical and Computer Engineering from Carnegie Mellon University in 1989. His research interests include various aspects of digital and analog integrated circuit design, and simulation, optimization and modeling of electric power systems. He was a co-founder of  Fabbrix Inc., Extreme DA, and Pearl Street Technologies. He has received various awards, including Westinghouse corporation’s highest engineering achievement award, the Semiconductor Research Corporation (SRC) Technical Excellence Awards in 1991 and 1999, the FCRP inaugural Richard A. Newton GSRC Industrial Impact Award, the SRC Aristotle award in 2008, the 2010 IEEE Circuits and Systems Society Mac Van Valkenburg Award, the ACM/IEEE A. Richard Newton Technical Impact Award in Electronic Design Automation in 2011, the Carnegie Institute of Technology B.R. Teare Teaching Award for 2013, and the 2015 Semiconductor Industry Association (SIA) University Researcher Award. He is a co-author of "Electronic Circuit and System Simulation Methods," McGraw-Hill, 1995 and "IC Interconnect Analysis," Kluwer, 2002. He has published over 400 conference and journal papers and holds 40 U.S. patents. He is a fellow of IEEE.
\end{IEEEbiography}

\end{document}